\documentclass{article}

     \usepackage[preprint]{neurips_2024}

\usepackage[utf8]{inputenc} 
\usepackage[T1]{fontenc}    
\usepackage{hyperref}       
\usepackage{url}            
\usepackage{booktabs}       
\usepackage{amsfonts}       
\usepackage{nicefrac}       
\usepackage{microtype}      
\usepackage{xcolor}         

\usepackage{subfig}
\usepackage{wrapfig}
\usepackage{amsmath}
\usepackage{amsthm}
\usepackage{mathtools,amssymb}
\usepackage{algorithm}
\usepackage{algorithmic}

\newtheorem{lemma}{Lemma}[section]
\newtheorem{thm}{\bf Theorem}[section]
\newtheorem{cor}{\bf Corollary}[section]
\newtheorem{definition}{\bf Definition}[section]

\newtheorem{remark}{\bf Remark}[section]

\title{Transductive Off-policy Proximal Policy Optimization}

\author{
  Yaozhong Gan \\
  yzgancn@163.com \\
   \And
    Xiaoyang Tan$^{{\ast}}$ \\
	 \\
	 \AND
   Zhe Wu \\
    \\
   \And
   Junliang Xing\thanks{corresponding author.} \\
    \\
	}

\begin{document}

\maketitle

\begin{abstract}
  Proximal Policy Optimization (PPO) is a popular model-free reinforcement learning algorithm, esteemed for its simplicity and efficacy. However, due to its inherent on-policy nature, its proficiency in harnessing data from disparate policies is constrained. This paper introduces a novel off-policy extension to the original PPO method, christened Transductive Off-policy PPO (ToPPO). Herein, we provide theoretical justification for incorporating off-policy data in PPO training and prudent guidelines for its safe application. Our contribution includes a novel formulation of the policy improvement lower bound for prospective policies derived from off-policy data, accompanied by a computationally efficient mechanism to optimize this bound, underpinned by assurances of monotonic improvement. Comprehensive experimental results across six representative tasks underscore ToPPO's promising performance. 
\end{abstract}

\section{Introduction}
Proximal Policy Optimization (PPO) \cite{Schppo} is a widely recognized, model-free reinforcement learning (RL) algorithm owing to its easiness to use in practice and its effectiveness. It has been successfully applied in many domains, such as Atari games \cite{Mni}, continuous control tasks \cite{Duan}, and robot control \cite{Lil}. 
The PPO algorithm can be considered a simplified version of the Trust Region Policy Optimization (TRPO) algorithm \cite{Schtrpo}, which is computationally more complex but offers better theoretical properties. However, both TRPO and PPO are on-policy algorithms, meaning they cannot use data from policies other than the underlying policy. This constraint potentially reduces their learning efficiency when off-policy data is available.

The above issue has recently drawn the attention of researchers, and several methods have been proposed. Typical algorithms include off-policy trust region policy optimization (off-policy TRPO) \cite{Wen}, generalized proximal policy optimization (GePPO) \cite{Que}, and so on. For example, GePPO addressed the off-policy issue by explicitly introducing a behavior policy in their objective.
Specifically, GePPO employs the V-trace technique \cite{Esp} to estimate the advantage function, which serves as a commonly used off-policy correction method. However, as demonstrated in \cite{Schm}, the advantage function generated by the V-trace method is biased. This bias occurs because its state-value function may converge to points that differ quite from the current policy, such as the local optimum region, thereby impacting the policy's performance.

Despite the above efforts, in practice, it is not uncommon for people to train on-policy algorithms such as PPO using data collected from the previous several rounds without performing any off-policy correction. 
For example, Hernandez-Garcia and Sutton \cite{Fer} ran the on-policy SARSA algorithm on off-policy data in the experience buffer and found no adverse effect on overall performance. Fujimoto
et al. \cite{Fuj2} demonstrated a high correlation between off-policy experiences and the current policy during environmental interaction. While this provides some intuitive explanation, there remains a notable gap between the empirical success and the underlying theoretical support. 

In this paper, we propose a new off-policy PPO algorithm with slight modifications to the original on-policy PPO. This modification enables PPO to utilize off-policy data in a theoretically sound manner, offering guidance on their safe usage. Inspired by transductive inference \cite{Joa,Bru}, our key focus is on estimating the desired advantage function for off-policy PPO directly from data generated by policies other than the current one, while ensuring monotonic policy improvement. By achieving this, we not only avoid the aforementioned bias issue but also enhance computational efficiency by eliminating the need to estimate the advantage function of the old policy in each iteration.

To this end, we introduce a novel formulation of the policy improvement lower bound for candidate policies. This formulation relies on the advantage function of the off-policy responsible for the data itself, thus mitigating bias problems. We optimize this performance bound by reformulating it as a constrained optimization problem. We solve it within the PPO computational framework, utilizing a modified clipping mechanism. 
And the method proposed in this paper can provide a theoretical explanation for the trick of reusing sample data in the PPO method \cite{Schppo}.
The resulting algorithm, Transductive Off-policy PPO (ToPPO), outperforms several closely related state-of-the-art off-policy algorithms on OpenAI Gym's MuJoCo environments \cite{Gre}.

\section{Preliminaries}\label{Backg}

Commonly, the reinforcement learning problem can be modeled as a Markov Decision Process (MDP), which is described by the tuple $\left\langle \mathcal{S}, \mathcal{A}, P, R, \gamma\right\rangle$ \cite{Sutton}.
$ \mathcal{S} $ and $ \mathcal{A} $ are the state space and action space respectively.
The function $ P(s'|s, a): \mathcal{S}\times\mathcal{A}\times\mathcal{S}\longmapsto [0, 1]$ is the transition probability function from state $ s $ to state $ s' $ under action $ a $.
The function $ R(s,a): \mathcal{S}\times\mathcal{A} \longmapsto \mathbb{R} $ is the reward function.
And $ \gamma \in [0, 1)$ is the discount factor for long-horizon returns.
In a state $ s $, the agent performs an action $ a $ according to a stochastic policy $ \pi: \mathcal{S}\times\mathcal{A} \longmapsto [0, 1] $ (satisfies $ \sum_{a}\pi(a|s)=1 $).
The environment returns a reward $ R(s, a) $ and a new state $ s' $ according to the transition function $ P(s'|s, a) $.
The agent interacts with the MDP to give a trajectory $ \tau $ of states , actions, and rewards, that is, $s_0, a_0, R(s_0, a_0), \cdots, s_t, a_t, R(s_t, a_t), \cdots $ over $ \mathcal{S}\times\mathcal{A}\times\mathbb{R} $ \cite{Sil}.
Under a given policy $ \pi $, the state-action value function and state-value function are defined as 
\begin{align*}
&Q^{\pi}(s_t, a_t) = \mathbb{E}_{\tau\sim\pi}[G_t|s_t, a_t],
V^{\pi}(s_t) = \mathbb{E}_{\tau\sim\pi}[G_t|s_t],
\end{align*}
where $ G_t = \sum_{i=0}^{\infty}\gamma^i R_{t+i} $ is the discount return, and $ R_t = R(s_{t}, a_{t}) $. 
It is clear that $ V^{\pi}(s_t) = \mathbb{E}_{a_t}Q^{\pi}(s_t, a_t) $.
Correspondingly, advantage function can be represented $ A^{\pi}(s, a) = Q^{\pi}(s, a)-V^{\pi}(s) $.
We know that $ \sum_{a}\pi(a|s)A^{\pi}(s, a)=0 $.

Let $ \rho^{\pi} $ be a normalized discount state visitation distribution, defined
$ \rho^{\pi}(s) = (1-\gamma)\sum_{t=0}^{\infty}\gamma^t\mathbb{P}(s_t=s|\rho_0, \pi) , $
where $ \rho_0 $ is the initial state distribution \cite{Kak}.
And the normalized discount state-action visitation distribution can be represented $ \rho^{\pi}(s,a) = \rho^{\pi}(s)\pi(a|s) $.
We make it clear from the context whether $ \rho^{\pi} $ refers to the state or state-action distribution.

The goal is to learn a policy that maximizes the expected total discounted reward $ \eta(\pi) $, defined
\begin{equation*}
\eta(\pi) = \mathbb{E}_{\tau\sim\pi}\left[\sum_{i=0}^{\infty}\gamma^i R(s_{i}, a_{i})\right].
\end{equation*}
The following identity indicates that the distance between the policy performance of $ \pi $ and $ \hat{\pi} $ is related to the advantage over $ \pi $ \cite{Kak}:
\begin{equation}\label{dif_pi_hatpi}
\eta(\hat{\pi}) = \eta(\pi)+
\frac{1}{1-\gamma}\mathbb{E}_{s\sim\rho^{\hat{\pi}}, a\sim\hat{\pi}}\left[A^{\pi}(s, a)\right].
\end{equation}

Note that if $ \mathbb{E}_{a\sim\hat{\pi}}\left[A^{\pi}(s, a)\right]>0 $ is satisfied at every state $ s $, any policy update $ \pi\rightarrow\hat{\pi} $ can guarantee the improvement of policy performance $ \eta $.
However, Eqn.(\ref{dif_pi_hatpi}) is difficult to optimize due to the unknown of $ \rho^{\hat{\pi}} $.

\subsection{Policy Improvement Lower Bound}
For policy improvement, some researchers study the lower bound of Eqn.(\ref{dif_pi_hatpi}) to optimize the current policy $ \pi $ \cite{Pir, Met,Schtrpo, Schppo, Wen, Que}. 
Some on-policy and off-policy algorithms are implemented based on the following lemma, which is a generalized form of CPI \cite{Kak}.

\begin{lemma}\label{PILB}
	(Policy Improvement Lower Bound) Consider a current policy $ \pi_{k} $, and any policies $ \pi $ and $ \mu $, we have
	\begin{align*}
	\eta(\pi)-\eta(\pi_k)\geq
	&
	\frac{1}{1-\gamma}\mathbb{E}_{(s,a)\sim\rho^{\mu}}\left[\frac{\pi(a|s)}{\mu(a|s)}A^{\pi_k}(s,a)\right]
	-
	\frac{4\epsilon\gamma}{(1-\gamma)^2}\delta_{\max}^{\pi_k, \pi}\cdot\delta^{\pi, \mu},
	\end{align*}
	where $\epsilon=\max _{s, a}\left| A^{\pi_k}(s, a)\right|$, $\delta^{\pi, \mu}=\mathbb{E}_{s \sim \rho^{\mu}} \mathbb{D}_{\mathcal{T} \mathcal{V}}(\mu, \pi)(s)$, and $ \delta_{\max}^{\pi_k, \pi} = \max_s \mathbb{D}_{\mathcal{T} \mathcal{V}}(\pi_k, \pi)(s)$. $ \mathbb{D}_{\mathcal{T} \mathcal{V}}(\pi_1, \pi_2)(s) = \frac{1}{2}\sum_{a}|\pi_1(a|s)-\pi_2(a|s)|$ represents the total variation distance (TV) between $ \pi_1(a|s) $ and $ \pi_2(a|s) $ at every state $ s $.
\end{lemma}

The proof of the lemma is given in Appendix.

Compared with Corollary 3.6 of the paper \cite{Pir} or Theorem 1 of the off-policy TRPO \cite{Wen}, the lower bound of Lemma \ref{PILB} is improved by replacing the maximum operator with expectation. Note that the first term of the lower bound is called the surrogate objective function, and the second term is called the penalty term. The right-hand side of Lemma \ref{PILB} reveals that improving the surrogate objective can guarantee the improvement of expected total discounted reward $ \eta $ \cite{Schtrpo}. This can be done by optimizing the lower bound by using a linear approximation of the surrogate objective and a quadratic approximation of the penalty term.

\paragraph{Generalized Proximal Policy Optimization}
However, it is complicated to optimize whether using a constraint term or a penalty term because the upper bound parameters and the penalty coefficients of the constraint need to be manually adjusted in different environments.
A generalized version of proximal policy optimization (GePPO) \cite{Que} is proposed and based on the policy improvement lower bound in Lemma \ref{PILB}.
By considering the following surrogate objective function at each policy update, GePPO ensures that the new policy is close to the current policy:
\begin{align*}
\!L(\pi)\!=&\mathbb{E}_{i\sim\nu}\mathbb{E}_{(s,a)\sim\rho^{\pi_{k-i}}}
\!\min\!\left(\!\frac{\pi(a|s)}{\pi_{k-i}(a|s)}A^{\pi_{k}}(s,a)\right.\!,
\left.\!\operatorname{clip}\!\left(\!\frac{\pi(a|s)}{\pi_{k-i}(a|s)}, l(s, a), u(s,a) \!\right)\!A^{\pi_{k}}(s,a)\!\right),
\end{align*}
where $ \text{clip}(x, a, b)=\min(\max(x, a), b) $ is a truncation function, $ l(s, a)=\frac{\pi_k(a|s)}{\pi_{k-i}(a|s)}-\epsilon $, $ u(s, a)=\frac{\pi_k(a|s)}{\pi_{k-i}(a|s)}+\epsilon $, distribution $ \nu=[\nu_0, \cdots, \nu_M] $ over the previous $ M $ policies, and $ \pi_k $ represents the current policy. The generalized clipping mechanism is obtained by constraining the total variation distance of $ \pi $, $ \pi_k $, and $ \pi_{k-i} $, but this mechanism has a small problem where $ l(s, a) $ may be less than zero. 
The above surrogate objective requires estimating the advantage function $ A^{\pi_k}(s, a) $ of the current policy using the previous trajectory samples. They use the V-trace technique \cite{Esp}. Unfortunately, as demonstrated in \cite{Schm}, the advantage function generated by the V-trace method is biased. This bias occurs because its state-value function may converge to points that are quite different from the current policy, such as the local optimum region, thereby impacting the policy’s performance, see Section \ref{ReGePPO} for a detailed discussion.

\section{Transductive Off-policy PPO}\label{off-trpo}
Given the inaccurate estimation of the advantage function $ A^{\pi_k} $ of the current policy $ \pi_k $ using the previous trajectory $ \tau\sim\mu$, our idea is whether we can utilize the advantage function $ A^{\mu} $ directly.
The answer is yes, but an additional condition is required.

First, we proposed a new surrogate objective function, defined as:
\begin{equation*}
\mathcal{L}_{\mu}(\pi) = \frac{1}{1-\gamma}\mathbb{E}_{(s,a)\sim\rho^{\mu}}\left[\frac{\pi(a|s)}{\mu(a|s)}A^{\mu}(s,a)\right].
\end{equation*}
The above formula does not suffer from the problem of the GePPO algorithm because this formula replaces $ A^{\pi_k} $ with $ A^{\mu} $.
Although this avoids the bias problem, an immediate question is whether it has good properties, \emph{i.e.}, performance bound. 
Next, a lower bound of policy performance about $ \mathcal{L}_{\mu}(\pi) $ is provided in the following lemma.

\begin{lemma}\label{PILB-v}
	(Lower Bound) Consider a current policy $ \pi_{k} $, and any policies $ \pi $ and $ \mu $, we have
	\begin{align}\label{pimu}
	\begin{aligned}
	\eta(\pi)-\eta(\pi_k)\geq
	&\mathcal{L}_{\mu}(\pi)-
	\frac{2(1+\gamma)\epsilon}{(1-\gamma)^2}
	\delta_{\max}^{\mu, \pi_k}
	-
	\frac{4\epsilon\gamma}{(1-\gamma)^2}
	\delta_{\max}^{\mu, \pi}\cdot\delta^{\mu, \pi},
	\end{aligned}
	\end{align}
	where $\epsilon=\max _{s, a}\left| A^{\mu}(s, a)\right|$, $\delta^{\mu, \pi}=\mathbb{E}_{s \sim \rho^{\mu}} \mathbb{D}_{\mathcal{T} \mathcal{V}}(\mu, \pi)(s)$, 
	$ \delta_{\max}^{\mu, \pi} = \max_s \mathbb{D}_{\mathcal{T} \mathcal{V}}(\mu, \pi)(s)$, and
	$ \delta_{\max}^{\mu, \pi_k} = \max_s \mathbb{D}_{\mathcal{T} \mathcal{V}}(\mu, \pi_k)(s)$. 
	$ \mathbb{D}_{\mathcal{T} \mathcal{V}}(\pi_1, \pi_2)(s) = \frac{1}{2}\sum_{a}|\pi_1(a|s)-\pi_2(a|s)|$ represents the total variation distance (TV) between $ \pi_1(a|s) $ and $ \pi_2(a|s) $ at every state $ s $.
\end{lemma}
%
The proof of the lemma is given in Appendix.

Compared with the Lemma \ref{PILB}, the lower bound of the Lemma \ref{PILB-v} has an extra term $ \delta_{\max}^{\pi_k, \mu} $, except for the advantage function $ A^{\mu} $.
This term arises by replacing the current advantage function $ A^{\pi_{k}} $ with $ A^{\mu} $.
Note that when $ \mu=\pi_{k} $, this lower bound is consistent with the Theorem 1 of TRPO \cite{Schtrpo}.

\begin{thm}\label{IMPI}
	(Monotonic Improvement)
	Consider the current policy $ \pi_k $, define
	\begin{equation*}
	F_{\pi_{k}}(\pi)=\mathcal{L}_{\pi_{k}}(\pi)-\frac{4\epsilon\gamma}{(1-\gamma)^2}\delta_{\max}^{\pi_k, \pi}\cdot\delta^{\pi_k, \pi}.
	\end{equation*}
	Assume $ \pi_{k+1}= \arg\max_{\pi} F_{\pi_{k}}(\pi)$ exists and satisfies $ F_{\pi_{k}}(\pi_{k+1})>0 $, there exists policy $ \mu $ and constant $ \alpha>0 $ that satisfies $ d(\mu, \pi_{k})<\alpha $, then
	\begin{equation*}
	\mathcal{L}_{\mu}(\pi_{k+1})-
	\frac{2(1+\gamma)\epsilon}{(1-\gamma)^2}
	\delta_{\max}^{\mu, \pi_k}-
	\frac{4\epsilon\gamma}{(1-\gamma)^2}
	\delta_{\max}^{\mu, \pi_{k+1}}\cdot\delta^{\mu, \pi_{k+1}}>0.
	\end{equation*}
\end{thm}
The proof of the lemma is given in Appendix.

Note that $ F_{\pi_{k}}(\pi) $ is the lower bound of TRPO.
From Theorem \ref{IMPI}, we know that there exists policy $ \mu $ that satisfies the monotonic improvement of policy performance.
Furthermore, we can see that our proposed method may give a better lower bound.
One can be see that if $ \alpha $ is equal to zero, it becomes an on-policy algorithm; if $ \alpha $ is large, 
it may not satisfy the policy improvement, so we need find a proper value for $\alpha$.
And it avoids estimating the advantage function of the current policy by using a variant of the importance sampling.
Next, since the specific form of $ \mu $ isn't known, we consider the parameterized policy $ \mu_{\phi}(a|s) $.
And we give a practical method by deducing the constrained optimization problem.

\paragraph{Derivation of the Constrained Optimization Problem}

To optimize the lower bound, given $ \pi_{\theta_k} $, we consider parameterized policy $ \pi_{\theta}(a|s) $, and $ \mu_{\phi}(a|s) $.
We evaluate the following maximization problem:
\begin{align*}
\underset{\phi, \theta}{\operatorname{maximize}}\ 
\mathcal{L}_{\mu_{\phi}}(\pi_{\theta})-
\frac{4\epsilon\gamma}{(1-\gamma)^2}
\left[\delta_{\max}^{\mu_{\phi}, \pi_{\theta_k}}+
\delta_{\max}^{\mu_{\phi}, \pi_{\theta}}\cdot\delta^{\mu_{\phi}, \pi_{\theta}}\right].
\end{align*}

When $ \mu_{\phi} = \pi_{\theta_{k}}$, the above formula reduces to the objective function of TRPO. Therefore, TRPO can be considered as a special case of our approach.
Better lower bounds may be obtained by optimizing this objective function, .
Next, a solution similar to TRPO \cite{Liu, Schtrpo} is adopted to optimize it.
The optimized step is very small because the penalty coefficient $ \frac{4\epsilon\gamma}{(1-\gamma)^2} $ recommended by the Lemma \ref{PILB-v} is too large.
And this penalty imposes trust region constraints that the TV distance is bounded at every state. This is impractical to solve because the state space is huge. Therefore, we use the average TV distance 
to approximate:
\begin{align}\label{ave_KL_ob}
\begin{aligned}
&\underset{\phi, \theta}{\operatorname{maximize}}\
\mathcal{L}_{\mu_{\phi}}(\pi_{\theta}),\\
&\text{subject to}\
\delta^{\mu_{\phi}, \pi_{\theta_k}}\leq\alpha_1, \
\delta^{\mu_{\phi}, \pi_{\theta}}\cdot\delta^{\mu_{\phi}, \pi_{\theta}}\leq\alpha_2,
\end{aligned}
\end{align}
where $\delta^{\mu_{\phi}, \pi_{\theta_k}}=\mathbb{E}_{s \sim \rho^{\mu_{\phi}}} \mathbb{D}_{\mathcal{T} \mathcal{V}}(\mu_{\phi}, \pi_{\theta_k})(s)$, and $\delta^{\mu_{\phi}, \pi_{\theta}}=\mathbb{E}_{s \sim \rho^{\mu_{\phi}}} \mathbb{D}_{\mathcal{T} \mathcal{V}}(\mu_{\phi}, \pi_{\theta})(s)$.

This constraint optimization problem can be solved approximately by using the alternating direction method of multipliers (ADMM) \cite{Che, Tao}, but it's still challenging to optimize $ \phi $ by Eqn.(\ref{ave_KL_ob}). To address this challenge, we give an approximate solution to avoid directly optimizing $ \phi $. For the first constraint term in Eqn.(\ref{ave_KL_ob}), we observe that the first term is an average TV distance between $ \mu_{\phi} $ and $ \pi_{\theta_k} $ and observe that $ \mu_{\phi}=\pi_k $ naturally holds. Our idea is not to optimize the parameter $ \phi $ directly and is to represent $ \mu_{\phi} $ with the previous policies $ \pi_{k-i} $, $ i = 0, \cdots, M $, satisfies $ M\leq k $. This way can make full use of off-policy data to learn the current policy. We must be aware that not all previous policies can represent $ \mu_{\phi} $ -- the constraint $ \delta^{\mu_{\phi}, \pi_{\theta_k}}\leq\alpha_1 $ must be satisfied.

Finally, we optimize the above constraint problem in two steps:
the first step is to select policies that satisfy the constraints from the last $ M $ policies, that is, 
$ \delta^{\mu_{\phi}, \pi_{\theta_k}}\leq\alpha_1 $, 
where $ \mu_{\phi}=\pi_{\theta_{k-i}} $, $ i=0, \cdots, M $, satisfies $ M\leq k $;
and for the second step, we optimize $ \theta $ directly by solving the following constraint optimization problem:
\begin{align}\label{fil_KL_ob}
\begin{aligned}
&\underset{\theta}{\operatorname{maximize}}\ \mathbb{E}_{(s,a)\sim\rho^{\mu}}\left[\frac{\pi_{\theta}(a|s)}{\mu(a|s)}A^{\mu}(s,a)\right],\\
&\text{subject to}\
\delta^{\mu_{\phi}, \pi_{\theta}}\leq\sqrt{\alpha_2}.\!
\end{aligned}
\end{align}
We know that $\delta^{\mu, \pi}=\mathbb{E}_{s \sim \rho^{\mu}} \mathbb{D}_{\mathcal{T} \mathcal{V}}(\mu, \pi)(s)\leq \sqrt{\mathbb{E}_{s \sim \rho^{\mu}} \mathbb{D}_{\mathcal{T} \mathcal{V}}^2(\mu, \pi)(s)}\leq\sqrt{\mathbb{E}_{s \sim \rho^{\mu}} \mathbb{D}_{KL}(\mu, \pi)(s)}$.
Thus, the constraint term can replace the TV distance with the Kullback-Leibler (KL) divergence \cite{Kul}. By constraining the upper bound, the TV distance is constrained. The above problem can be solved by a linear approximation of the surrogate objective function and a quadratic approximation of the constraint. Similar policy updates have been proposed in previous work \cite{Ach, Wen, Schtrpo}.

\subsection{The Clipped Surrogate Objection}\label{to-ppo}

In the previous section, when we optimize the problem (\ref{fil_KL_ob}), it faces a serious problem in the second step: we need to store the $ M $ previous policy $ \mu $ network parameters, 
which is because if we optimize the problem (\ref{fil_KL_ob}) directly, this will be optimized in the same way as TRPO, using a linear approximation of the surrogate objective and a quadratic approximation of the penalty term. When we compute a quadratic approximation $ \frac{\partial}{\partial \theta_{i}} \frac{\partial}{\partial \theta_{j}} \mathbb{E}_{s \sim \rho{\pi}} [D_{\mathrm{KL}}(\pi(\cdot|s, \theta_{k-i}) | \pi(\cdot|s, \theta))]|_{\theta=\theta_{k-i}} $, we need to keep the network parameters of $ \pi(\cdot|s,\theta_{k-i}) $. If there are many previous policies, it needs to retain many networks and takes up a lot of computer memory. 
It is therefore impractical to solve this problem. Inspired by PPO \cite{Schppo}, a practical variant of TRPO, we propose a new clipped surrogate objection according to Eqn.(\ref{fil_KL_ob}), defined as

\begin{equation}\label{of_ppo}
\begin{aligned}
L(\pi)=&\mathbb{E}_{(s,a)\sim \rho^{\mu}}
\min\left(\frac{\pi(a|s)}{\mu(a|s)}A^{\mu}(s,a)\right.,
\left.\text{clip}\left(\frac{\pi(a|s)}{\mu(a|s)}, l(s, a), u(s,a) \right)A^{\mu}(s,a)\right),
\end{aligned}
\end{equation}
where $ \mu(a|s)=\pi_{k-i}(a|s) $, and $ l(s, a) $ and $ u(s, a) $ are the clipping of probability ratio $ \frac{\pi(a|s)}{\mu(a|s)} $ lower and upper bounds, respectively.

In this way, we do not need to store the last $ M $ policy $ \pi_{k-i} $ network parameters and only need to save the probability value $ \mu(a|s) $ of the corresponding action $ a $ under state $ s $ in practice. Thus it becomes a very practical version. Note that we do not define $ l(s, a) $ and $ u(s, a) $. From the research of PPO \cite{Schppo, Que}, we find that the choice of the lower bound and the upper bound is very important, which connects to the policy's performance.

\begin{definition}\label{pi_k_i_pi}
	Consider a current policy $ \pi_k $ and clipping parameter $ \epsilon $,  and for any previous policies $ \pi_{\theta_{k-i}} $, $ i=0,\cdots,M $, the lower and upper bounds are defined as
	\begin{align}\label{up_lo}
	\begin{aligned}
	l(s,a) = \max(\frac{\pi_{k}(a|s)}{\pi_{k-i}(a|s)}-\epsilon, 0),\
	u(s,a) = \frac{\pi_{k}(a|s)}{\pi_{k-i}(a|s)}+\epsilon.
	\end{aligned}
	\end{align}
\end{definition}

Notice that the upper and lower bounds we defined are different from those in paper \cite{Que}. The main difference is in the lower bound, where we added a max function. This will avoid cases where the lower bound is less than zero. Thus, no incorrect optimization policies will be generated when the advantage function is less than zero.
In a heuristic way, these lower and upper bounds are related to the previous policy $ \pi_{k-i} $ and the current policy $ \pi_k $. 
Sampling state-action pairs from the state-action visitation distribution $ \rho^{\pi_{k-i}} $ can be simply viewed as an unbiased estimate of the above formulas.
This clipping mechanism removes the incentive for $ \frac{\pi}{\pi_{k-i}} -\frac{\pi_k}{\pi_{k-i}} $ to exceed $ \epsilon $. This can be simply viewed as an off-policy clipping mechanism.

Combining Eqn.(\ref{of_ppo}) and Eqn.(\ref{up_lo}), we present the Transductive Off-policy Proximal Policy Optimization algorithm (ToPPO), a practical variant that uses off-policy data and avoids the problem of storing the $ M $ previous policy network parameters faced by the constraint problem (\ref{fil_KL_ob}):
\begin{multline}\label{off_ppo_loss}
L_k(\pi)= \mathbb{E}_{(s,a) \sim \rho^{\pi_{k-i}}} \left[ \min \left( \frac{\pi(a|s)}{\pi_{k-i}(a|s)} A^{\pi_{k-i}}(s,a), \right. \right. \\ \left. \left.\operatorname{clip}\left( \frac{\pi(a|s)}{\pi_{k-i}(a|s)},\max(\frac{\pi_k(a|s)}{\pi_{k-i}(a|s)}-\epsilon, 0),\frac{\pi_k(a|s)}{\pi_{k-i}(a|s)}+\epsilon \right) A^{\pi_{k-i}}(s,a) \right) \right]. \ \ \
\end{multline}
This is the optimization objective function for the $ k $-th update. Algorithm 1 (due to space limitations, see the appendix) shows the detailed implementation pipeline. In each iteration, the ToPPO algorithm is divided into four steps: collect samples, update the policy network, select policies. 
The third step of Algorithm 1 are described in detail below.

\paragraph{Selecting policies} 
The third step is policy selection. According to the formula (\ref{ave_KL_ob}), we must first choose a suitable previous policy $ \mu $ to satisfy $ \delta^{\mu, \pi_{\theta_{k}}}\leq\alpha $, where $ \mu \in M $, and $ \alpha $ is the filter boundary. If $ \delta^{\mu, \pi_{\theta_{k}}}>\alpha $, we will delete the trajectories of $ \mu $ from $ M $. Otherwise, we will keep them. We then define a maximum length $ N $ of $ M $, \emph{i.e.}, $ N=|M| $. If the newly added sample in $ M $ exceeds the maximum length $ N $, the oldest policy is deleted, and the latest policy is kept. This approach is beneficial to the stability of the training progress. And the choice of the filter boundary $ \alpha $ is particularly important, it affects the performance of the algorithm, that is, if $ \alpha $ is equal to zero, it becomes an on-policy algorithm; if $ \alpha $ is large, 
it may not satisfy the policy improvement. Therefore, we need to choose a suitable $ \alpha $ value. 
Note that unlike DISC \cite{Seu}, which heuristically reuses old samples only constrains the upper bound of the ratio, but selecting policy method is theoretically guaranteed in this paper. Furthermore, this method can improve the stability of the overall training progress.
The reason why the selecting policies are placed in the third step is that the first iteration does not need to select.

\begin{remark}
	The selecting policies step is relevant to the theory presented in this paper without any heuristic elements. Although the practical version of the method proposed in this paper in Eqn. (\ref{off_ppo_loss}) is similar to the GePPO, the perspective of considering the problem is completely different. In addition to having good monotonicity, the method proposed in this paper explains a trick of the PPO algorithm very well in section \ref{Ana}, which is not available in the GePPO. 
\end{remark}

In appendix, we give a way to choose the clipping parameter $ \epsilon^{\text{ToPPO}} $ in Eqn. (\ref{off_ppo_loss}). 
It is worth noting that this paper only gives a lower threshold for the $ \epsilon^{\text{ToPPO}} $ value that can be chosen, due to the fact that the algorithm in this paper includes a step to select policy, which lead the size of the replay buffer $ M $ to be dynamically changing.
Thus the $ \epsilon^{\text{ToPPO}} $ value is also dynamic. It has been found experimentally that better results are achieved in some environments if the $ \epsilon^{\text{ToPPO}} $ value is dynamic. But this could bring some instability. Therefore, the main experiments are conducted with fixed the $ \epsilon^{\text{ToPPO}} $ value.

\section{Discussion}\label{Ana}

\subsection{Reanalyze the PPO Algorithm}

To achieve better performance, the PPO algorithm \cite{Schppo} uses a trick in its implementation, that is, it uses the sample data collected by the current policy several times to optimize the policy. This trick can be explained by the intuition that since the sample data is used to optimize the policy in each iteration, the new policy is not so far away from the current one. According to the theory of TRPO, the data can also be used to optimize the policy if the distance between the old and new policy isn't far away. When the policy is optimized by using the current samples, the current policy $ \pi_{\theta_{k}} $ will become $ \pi_{\theta_{k, 1}} $. And then the parameters of the policy $ \pi_{\theta_{k, 1}} $ will be optimized by reusing the sample data of the current policy. But one point to keep in mind is that the samples are still generated by the current policy $ \pi_{\theta_{k}} $ interacting with the environment rather than the policy $ \pi_{\theta_{k, 1}} $. Therefore, this update is a slightly different from the theory of TRPO \cite{Schtrpo} and is not exactly equivalent. Overall, the PPO algorithm leverages the trick of repeatedly using the data from the current policy during each iteration, that is, $ \pi_{\theta_{k}}\rightarrow \pi_{\theta_{k, 1}}\rightarrow\pi_{\theta_{k, 2}}\rightarrow\cdots\rightarrow\pi_{\theta_{k, N}}=\pi_{\theta_{k+1}} $. 

However, the approach proposed in this paper can provide a theoretical explanation that the PPO algorithm is reasonable in this way. Since the surrogate objection defined in this paper is shown in Eqn.(\ref{off_ppo_loss}), our algorithm can use the historical policy $ \pi_{\theta_{k-i}} $ data to optimize the policy, and the advantage function of surrogate function is for the policy $ \pi_{\theta_{k-i}} $. This way suggests that the data of the policy $ \pi_{\theta_{k-i}} $ can be used several times to optimize the policy, and is consistent with the trick used by PPO.

We analyze the PPO algorithm. By optimizing the surrogate function of PPO to obtain the new policy $ \pi_{\theta_{k+1}} $, we show that 
$$ \mathbb{E}_{s\sim\rho^{\pi_{\theta_{k}}}} V^{\pi_{\theta_{k+1}}}(s) \geq \mathbb{E}_{s\sim\rho^{\pi_{\theta_{k}}}} V^{\pi_{\theta_{k}}}(s).$$ 
The proof is given in the appendix. Given the current policy, the expectation of the value function of the new policy $ \pi_{\theta_{k+1}} $ is higher than the current $ \pi_{\theta_{k}} $. In other words, the value function is increasing on average, but may not be necessarily in every state. If the policy can be represented in tabular form, then it holds that $ V^{\pi_{\theta_{k+1}}}\geq V^{\pi_{\theta_{k}}} $, because the policy can be optimized in every state. It does not prove the PPO algorithm completely but provides insight into the effectiveness of the PPO algorithm from a certain perspective. Again, this inequality can provide another point of view for estimating the value function in offline reinforcement learning \cite{Kum}. The increase in the value function may be inevitable. So, it may be bad to be too conservative in estimating the value function and should be as mild as possible \cite{Lyu, Nak}.

\subsection{Reanalyze the GePPO Algorithm}\label{ReGePPO}

\begin{wrapfigure}[]{r}{0.35\textwidth}
	\vskip -5pt
	\centering
	\subfloat{\includegraphics[width=0.35\textwidth]{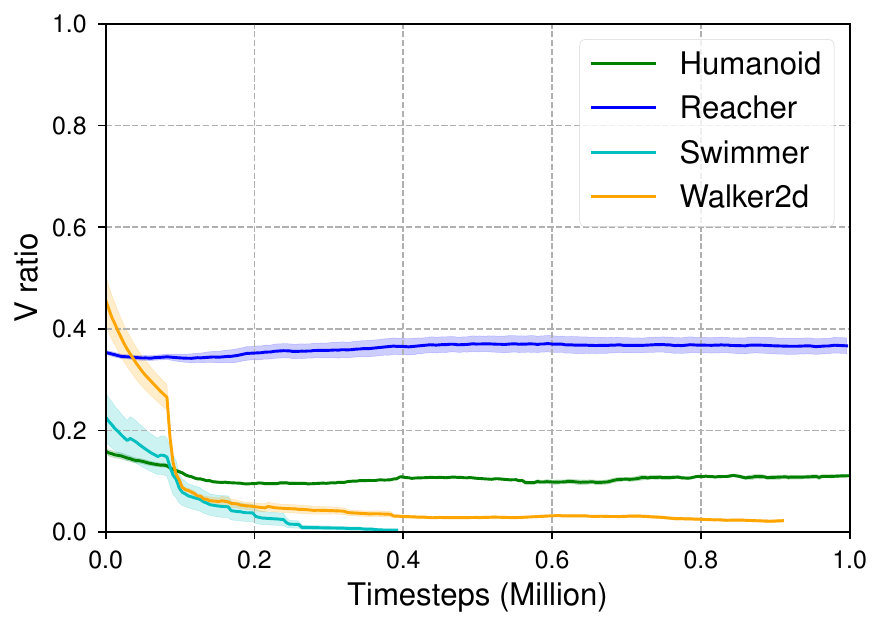}}
	\caption{The ratio of the difference between the true $ V^{\pi} $ values and the estimated $ \hat{V} $ values using V-trace technique to the true $ V^{\pi} $ values, \emph{i.e.} $  \mbox{V ration}=|\frac{V^{\pi}-\hat{V}}{V^{\pi}}| $.}
	\label{V_ratio}
	\vspace{-15pt}
\end{wrapfigure}

The GePPO algorithm requires estimating the advantage function $ A^{\pi_k}(s, a) $ of the current policy by the previous trajectory samples $ \tau\sim\mu $.
They use the V-trace technique \cite{Esp} by $V^{\pi_{\bar{\rho}}}\left(s_{t}\right)=V\left(s_{t}\right)+\sum_{k=0}^{K-1} \gamma^{k}\left(\prod_{i=0}^{k-1} c_{i}\right) \rho_{t} \delta_{t+k} V$, where $ \delta_{t+k} V= R_t +\gamma V(s_{t+1})-V(s_t) $ is the temporal difference error, $ \rho_t = \min\left(\frac{\pi_t}{\mu_t}, \bar{\rho}\right) $, and $ c_t=\min\left(\frac{\pi_t}{\mu_t}, \bar{c}\right) $ are truncated importance sampling (IS)
weights.
$ \pi_{\bar{\rho}} $ is defined as
\begin{equation*}
\pi_{\bar{\rho}}(a|s)=\frac{\min \left(\bar{\rho} \mu(a|s), \pi(a|s)\right)}{\sum_{b \in \mathcal{A}} \min \left(\bar{\rho} \mu(b|s), \pi(b|s)\right)}.
\end{equation*}

Note that the estimation of the value $ V^{\pi_{\bar{\rho}}} $ by the V-trace is biased, because the biased policy $ \pi_{\bar{\rho}} $ is very different from $ \pi $. Additionally, the estimated advantage function $ A^{\pi_k}(s, a) $ is also biased. As an illustrative example \cite{Schm}, consider two policies over a set of two actions, \emph{e.g.}, “left” and “right” in a tabular case. For any suitable small $ \phi\leq 1 $, define $ \mu=(\phi, 1-\phi) $ and $ \pi=(1-\phi, \phi) $, we see that when $ \phi\rightarrow 0 $, $ \pi $ and $ \mu $ become more focused on one action, then they rarely share trajectories. When $ \bar{\rho}=1 $, we see that $ \pi_{\bar{\rho}} $ is a uniform distribution. The V-trace estimate $ V^{\pi_{\bar{\rho}}} $ would calculate the average value of “left” and “right”, but this poorly represents the $ V^{\pi_{k}} $. And they show that the algorithm can get stuck in local optima and affect the policy's performance.
As a result, GePPO and off-policy TRPO may face the same problem. 

We also calculate their difference in several environments in the GePPO algorithm. In each iteration, given the state $ s $, the estimated value $ \hat{V}(s) $ is calculated by the V-trace technique. On the other hand, taking the state $ s $ as the starting point of the environment, the estimate of the true value $ V^{\pi}(s) $ is calculated from the trajectories generated by interacting the policy with the environment.
From Figure \ref{V_ratio}, the curve describes the ratio of the difference between the true $ V^{\pi} $ values and the estimated $ \hat{V} $ values using the V-trace technique to the true $ V^{\pi} $ values. One can be seen that the larger the ratio, the more significant the difference between them. 
This figure show that the $ V $ values of the current policy estimated by V-trace technique are inaccurate.
Furthermore, the estimated advantage function $ A $ of the current policy may also be inaccurate, which inevitably introduces bias and affects the policy's performance (see Fig. \ref{performance}).

\section{Experiments}\label{expe}
In this section, we present our experimental results to verify the effectiveness of the proposed ToPPO method on six continuous control tasks 
from the MuJoCo environments \cite{Tod} and some Atari games  \cite{Gre}.
Algorithm 1 (due to space limitations, see the appendix) gives the detailed implementation pipeline.
We conduct all the experiments mainly based on the code from Queeney et al. \cite{Que}.
For all methods, we use the same neural network architecture \cite{baselines, Hen}.
Our proposed method experiments with several closely related algorithms which have policy improvement guarantees, \emph{i.e.}, TRPO \cite{Schtrpo}, PPO \cite{Schppo}, DISC \cite{Han}, Off-policy TRPO (OTRPO) \cite{Wen}, Off-policy PPO (OPPO) \cite{Men}, GePPO \cite{Que}. See the appendix for details of the experimental setup.

\subsection{Evaluation}
\paragraph{Performance improvement} 
We evaluate the proposed ToPPO method. Figure \ref{performance} shows all the results of each algorithm. We observe from the figure that compared to PPO, ToPPO can improve the sample efficiency except for Walker2d, and it is the same as GePPO. This is because our method is able to utilize the off-policy data to update policy. We also see that GePPO can't improve the sampling efficiency in some environments, such as HumanoidStandup. According to the previous theory in Section \ref{Backg}, we infer that the estimated value of the advantage function is biased by using the truncated importance sampling from Figure \ref{V_ratio}, resulting in the poor performance of GePPO. Compared to GePPO, ToPPO does not introduce bias, and can improve performance. Although the difference between our method and GePPO seems small, there is a large gap in the mechanism, which is probably the reason for the good performance of our method.

\begin{figure*}[t]
	\centering
	\subfloat[HalfCheetah]{\includegraphics[width=0.24\textwidth]{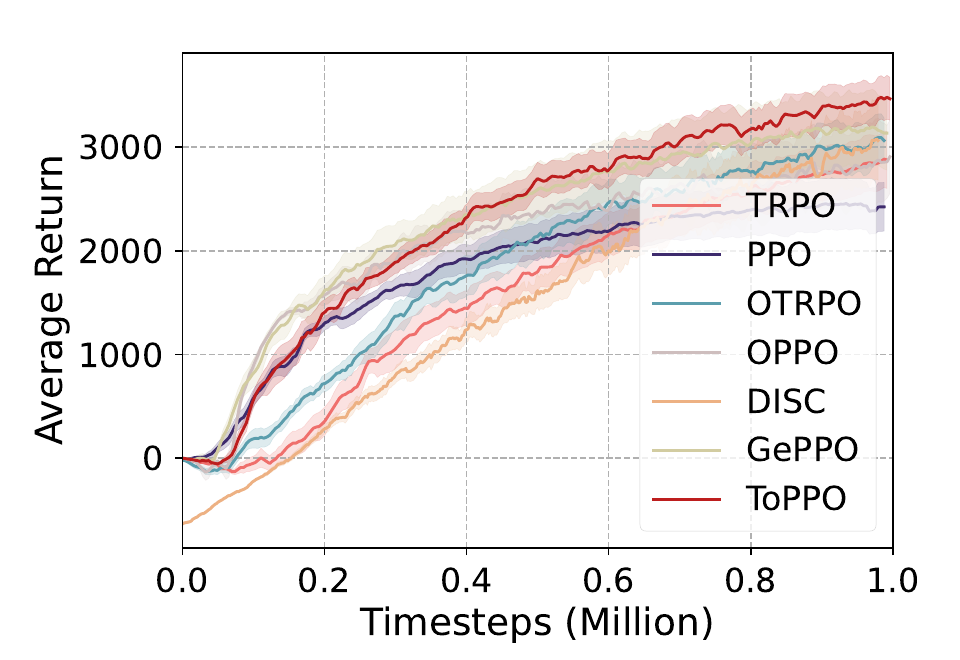}}
	\subfloat[Swimmer]{\includegraphics[width=0.24\textwidth]{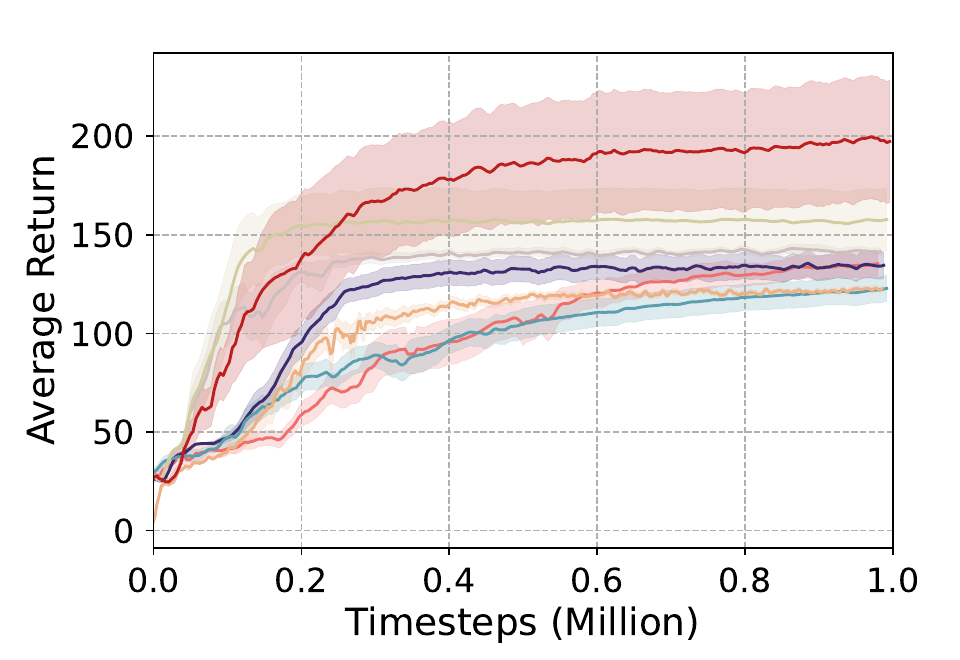}}
	\subfloat[Reacher]{\includegraphics[width=0.24\textwidth]{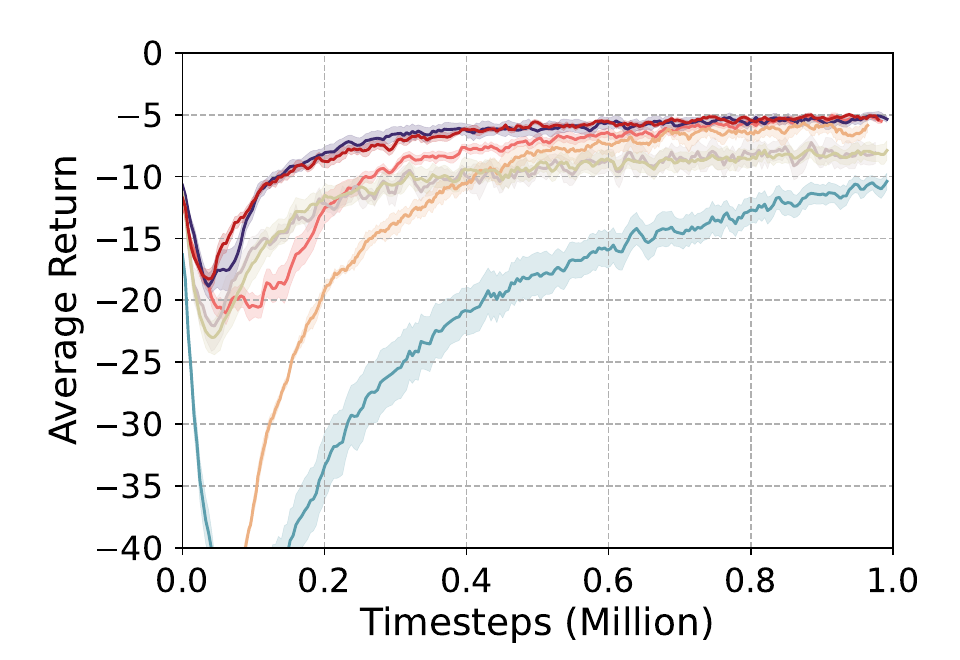}}
	\subfloat[Hopper]{\includegraphics[width=0.24\textwidth]{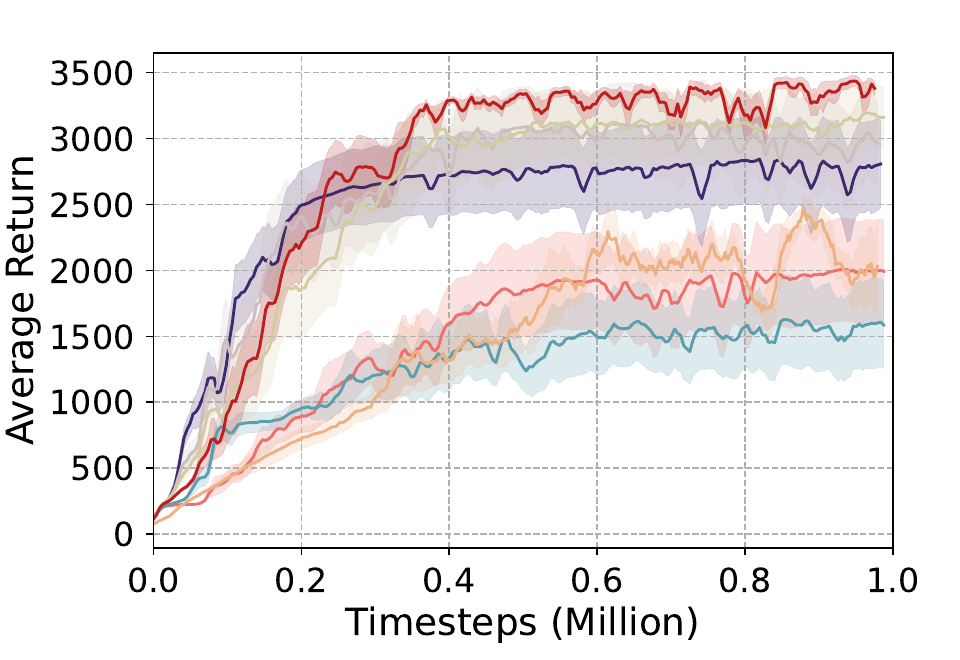}}\\
	\subfloat[HumanoidStandup]{\includegraphics[width=0.24\textwidth]{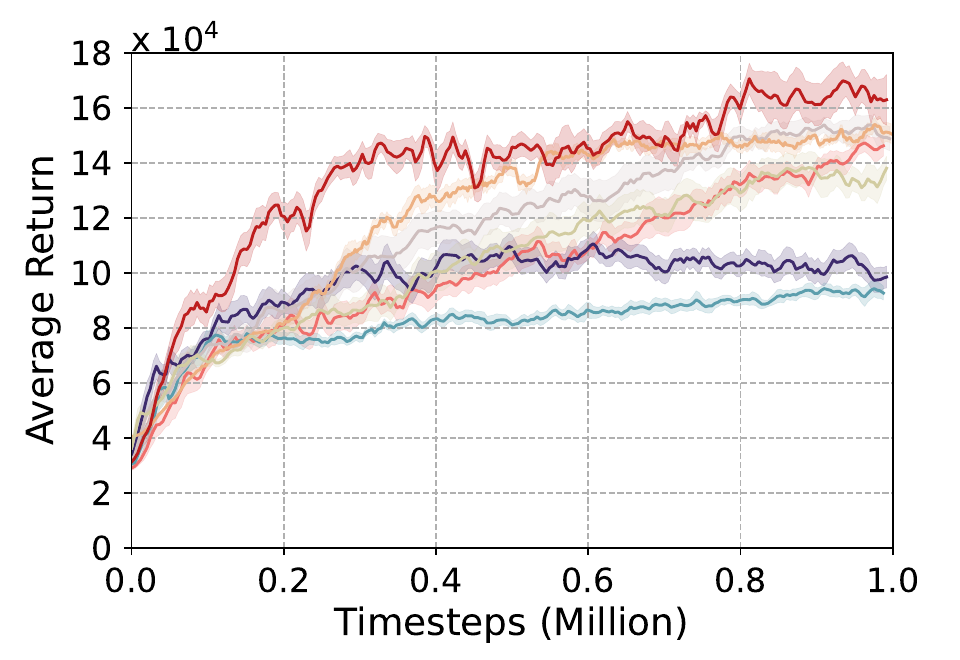}}
	\subfloat[Walker2d]{\includegraphics[width=0.24\textwidth]{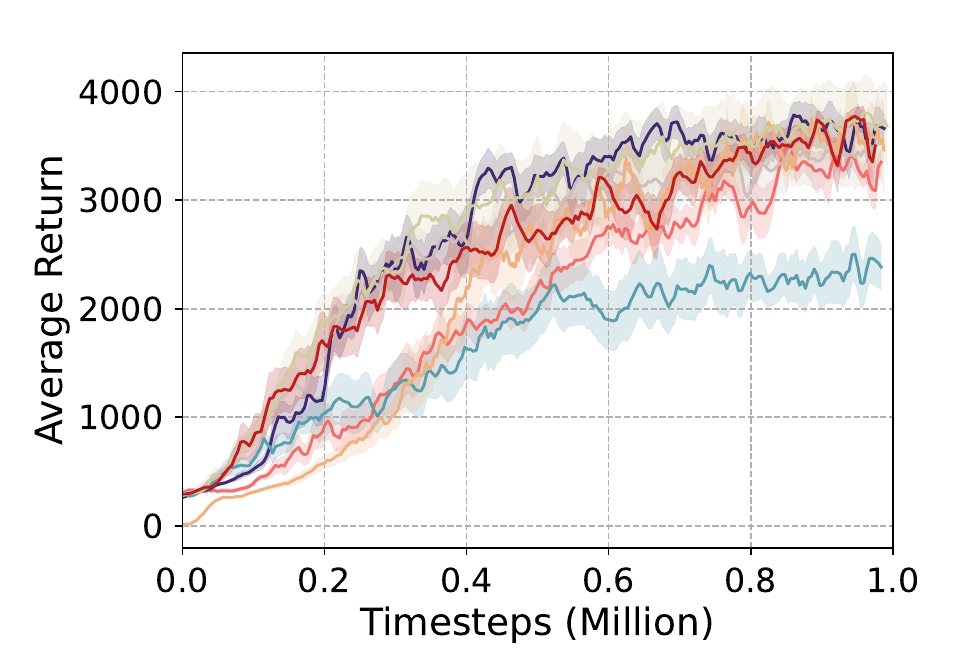}}
	\subfloat[Humanoid]{\includegraphics[width=0.24\textwidth]{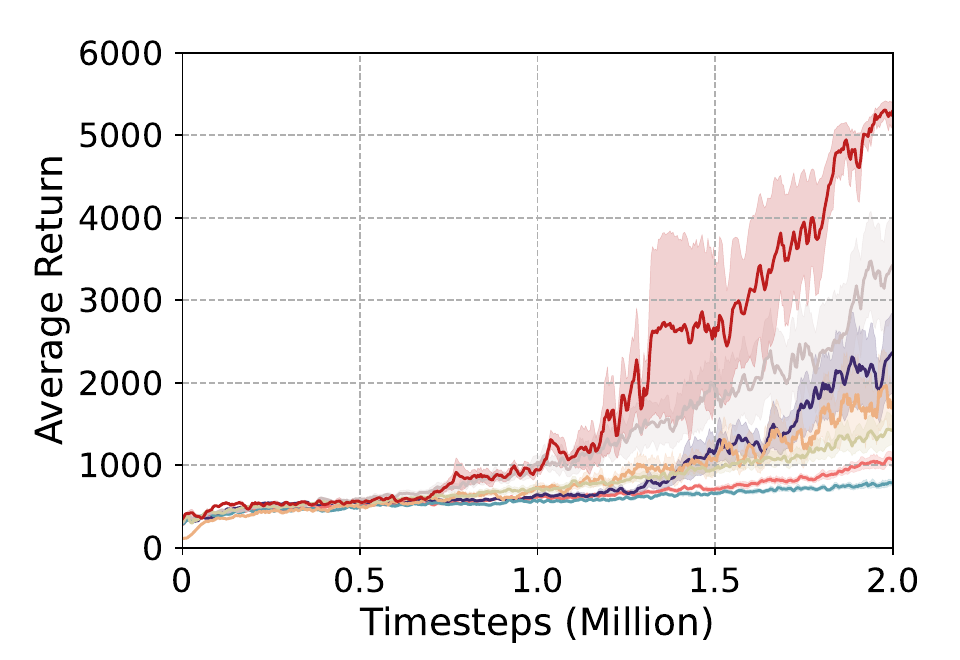}}
	\caption{Learning curves on the MuJoCo environments. Performance of \textit{ToPPO} vs. \textit{PPO}, \textit{OTRPO},  \textit{TRPO}, \textit{DISC}, \textit{OPPO}, and \textit{GePPO}.
		The shaded region indicates the standard deviation of ten random seeds.
		The X-axis represents the timesteps in the environment. 
	}
	\label{performance}
\end{figure*}

In the discrete environments, we randomly chose some Atari games.
The results are averaged over three seeds during 25M timesteps. We run our experiments across three seeds with fair evaluation metrics. We use the same hyperparameters $ \epsilon=0.1 $ and do not fine-tune them. Since none of the off-policy versions of PPO are playing in Atari games, we only compare the experiment with PPO. From Figure \ref{performance_atari}, this shows that  on more complex environments, our method obtains better results using off-policy data. Therefore, the ToPPO of our proposed has better sample efficiency. 
Please refer to the appendix for some additional results (refer to Figure \ref{performance_atari1})).

\paragraph{Ablation Studies}
In addition, we also verify the necessity of the constraints of policy selection. According to Theorem \ref{IMPI}, constraints on the previous policy $ \mu $ and the current policy $ \pi_k $ can guarantee the monotonic improvement of policy performance. Figure \ref{Compa1} shows the results of the two comparisons (\textit{$ N=5 $} vs. \textit{$ N=5 $ NOT} and \textit{$ N=6 $} vs. \textit{$ N=6 $ NOT}). `\textit{NOT}' is that our method removes the constraint of selecting policies, and conducts experiments. Compared to \textit{$ N=5 $} vs. \textit{$ N=5 $ NOT} or \textit{$ N=6 $} vs. \textit{$ N=6 $ NOT}, we see that selecting a policy to update by constraints can improve the performance of the algorithm in most control environments. Moreover, one can see that ToPPO may achieve better results if we adjust the parameter $ N $, such as HalfCheetah.

\begin{wrapfigure}[]{r}{0.42\textwidth}
	\centering
	\subfloat[HalfCheetah]{\includegraphics[width=0.2\textwidth]{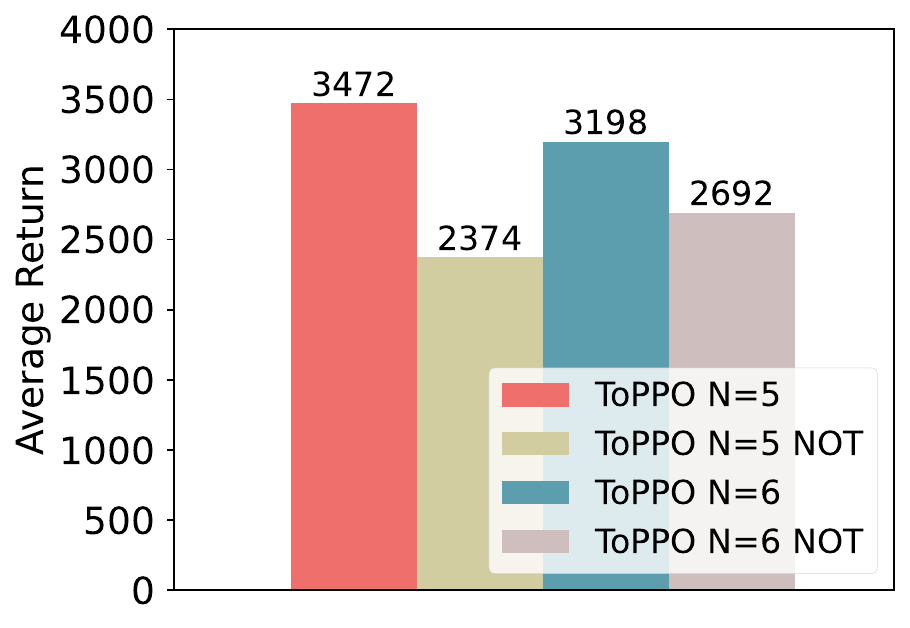}}
	\subfloat[HumanoidStandup]{\includegraphics[width=0.2\textwidth]{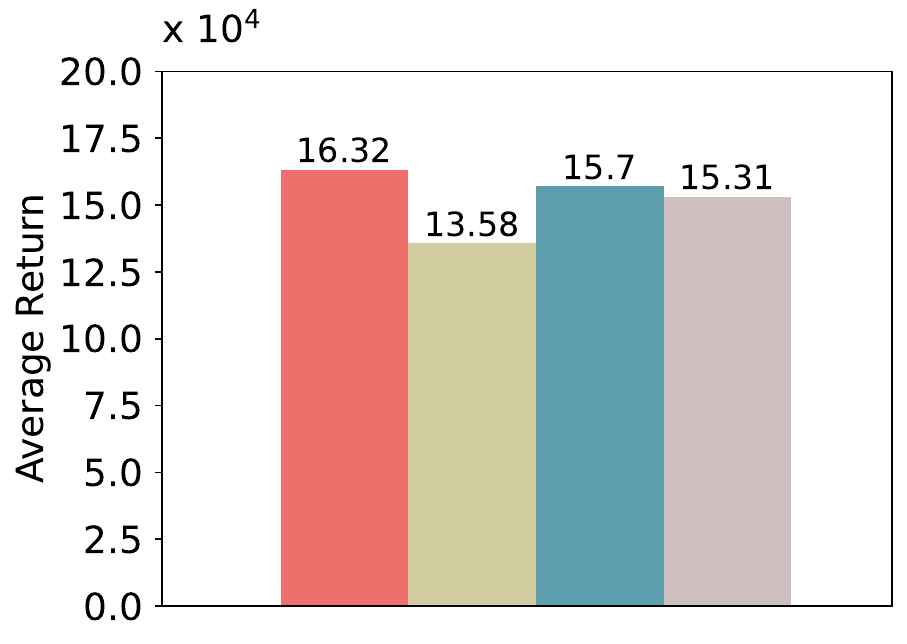}}
	\caption{Final performance of \textit{ToPPO} vs. \textit{ToPPO NOT} (remove the constraints of selecting policies)}
	\label{Compa1}
	\vspace{-5pt}
\end{wrapfigure}

\begin{figure*}[t]
	\centering
	\subfloat[BattleZone]{\includegraphics[width=0.24\textwidth]{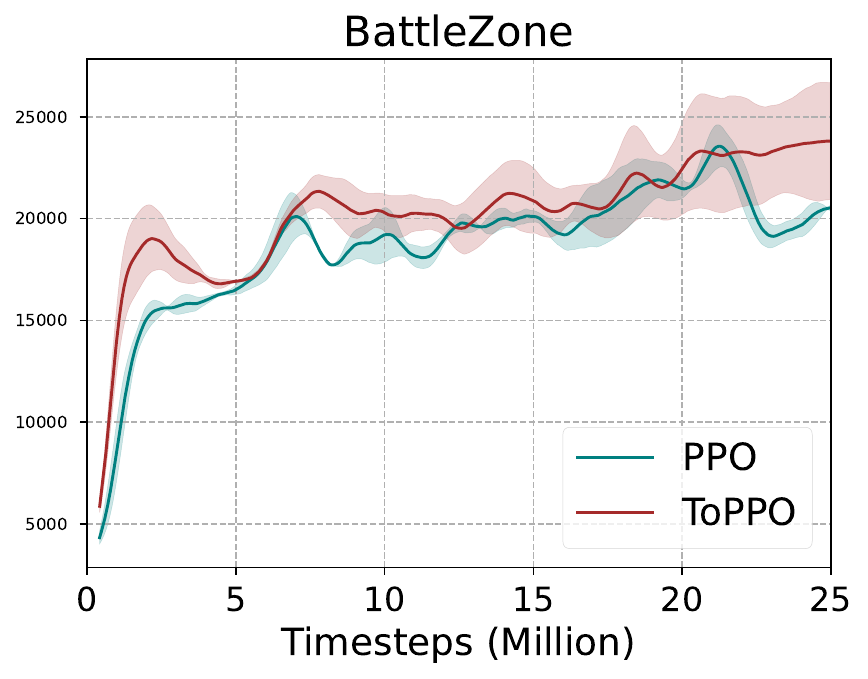}}
	\subfloat[Breakout]{\includegraphics[width=0.24\textwidth]{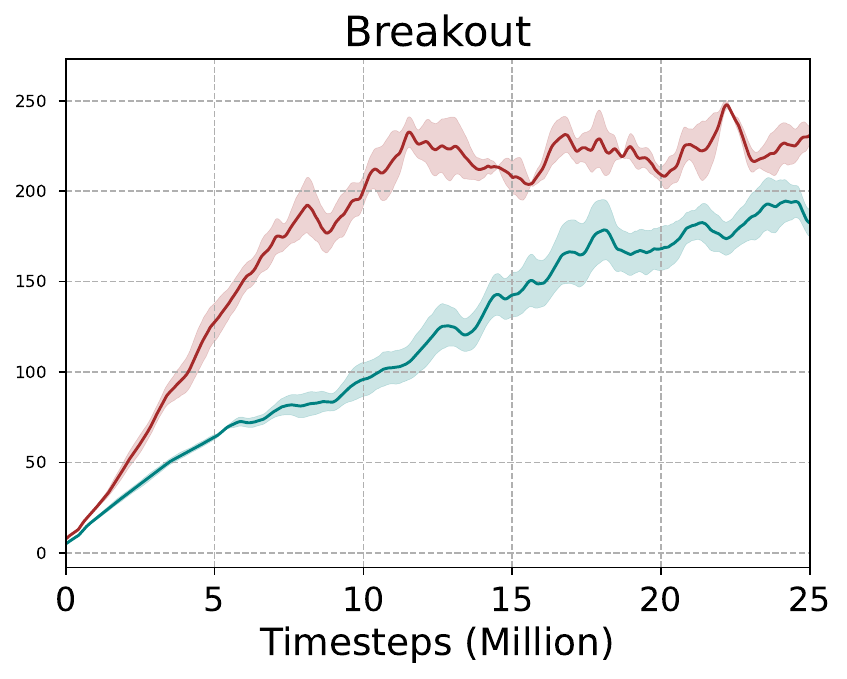}}
	\subfloat[Carnival]{\includegraphics[width=0.24\textwidth]{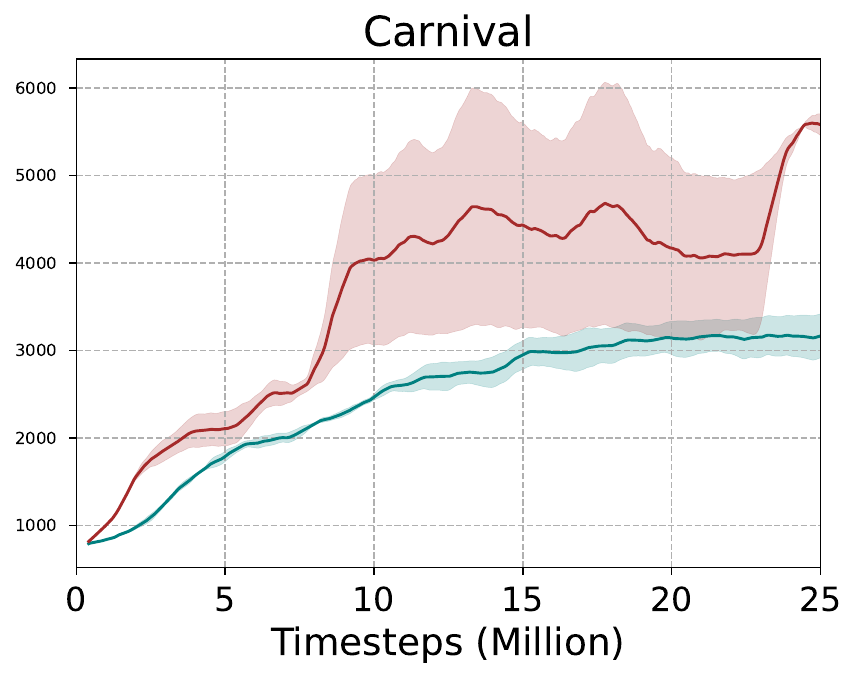}}
	\subfloat[CrazyClimber]{\includegraphics[width=0.24\textwidth]{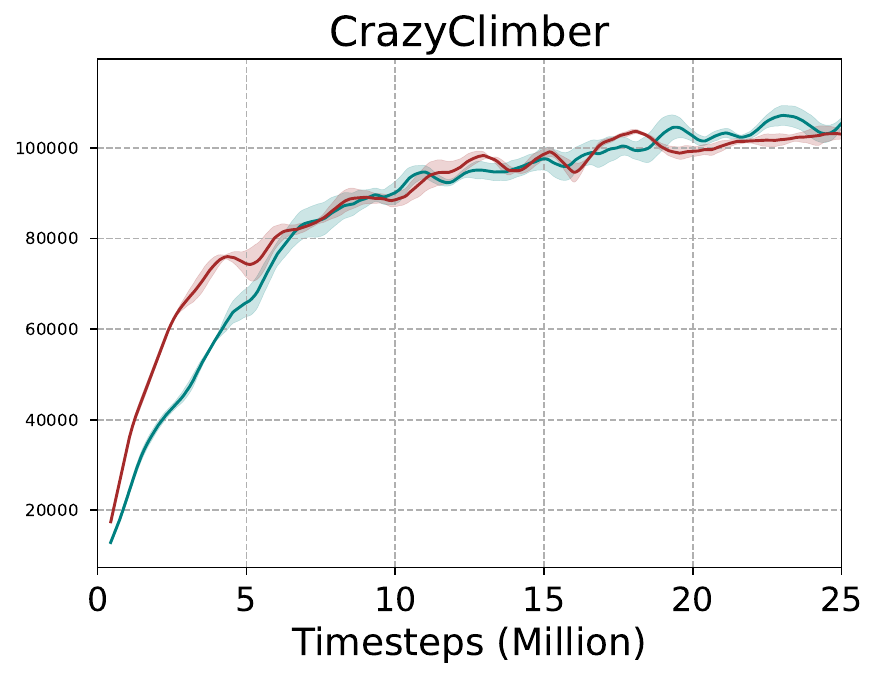}}\\
	\subfloat[Enduro]{\includegraphics[width=0.24\textwidth]{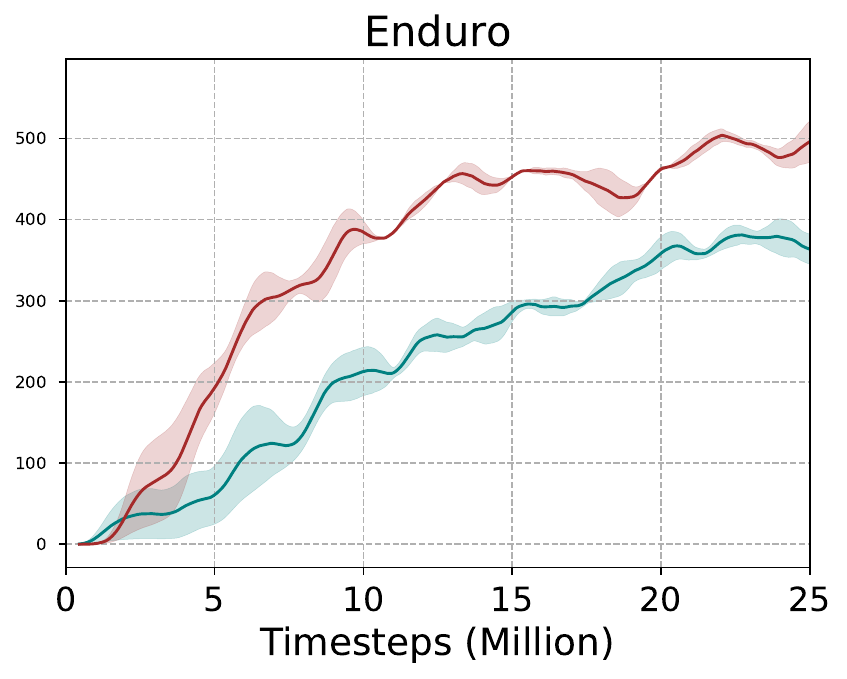}}
	\subfloat[FishingDerby]{\includegraphics[width=0.24\textwidth]{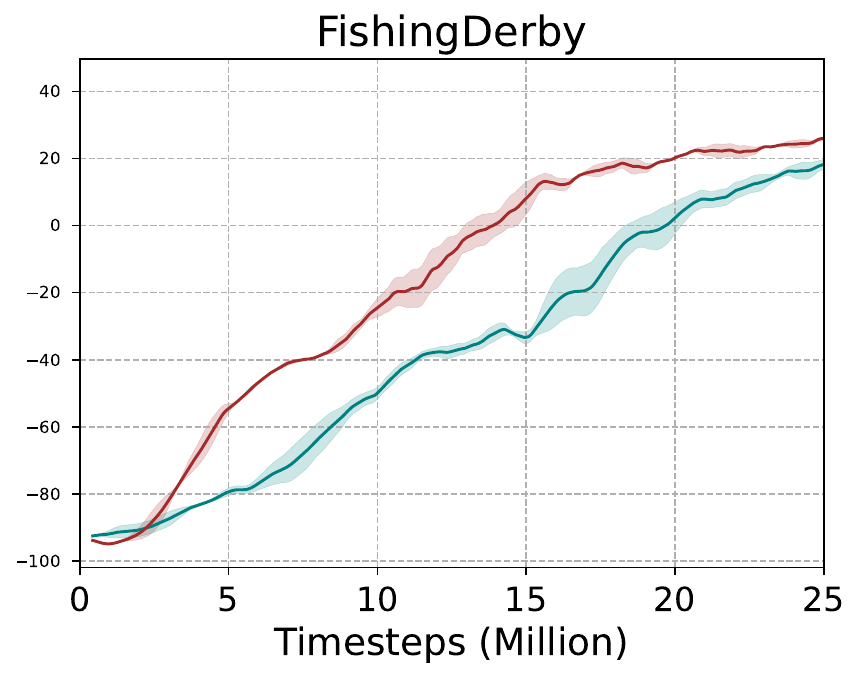}}
	\subfloat[Krull]{\includegraphics[width=0.24\textwidth]{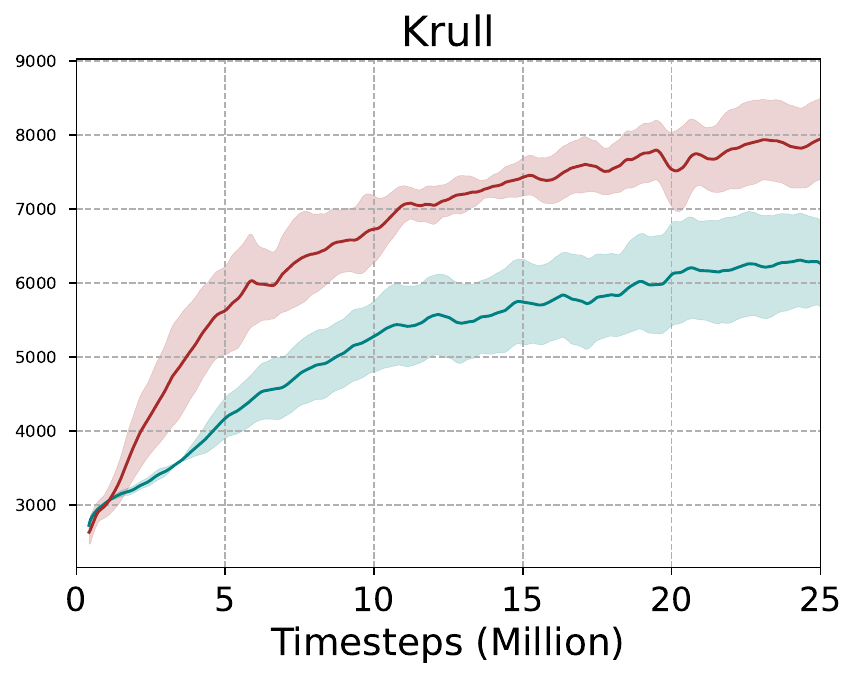}}
	\subfloat[NameThisGame]{\includegraphics[width=0.24\textwidth]{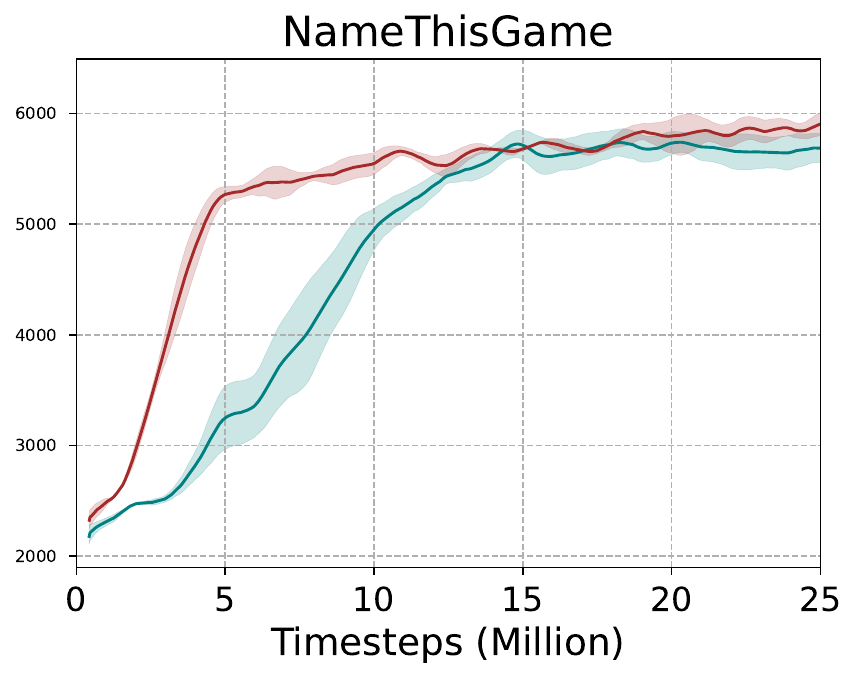}}
	\caption{Learning curves on the Atari environments. Performance of \textit{ToPPO} vs. \textit{PPO}.
		The shaded region indicates the standard deviation of three random seeds.
	}
	\label{performance_atari}
\end{figure*}

Next, we study the number of policies in $ M $ during the iteration. Due to removing the constraints of the policy selection, it is obvious that the length of $ M $ in the algorithm \textit{ToPPO $ N=5 $ NOT} is 5. If we keep the constraints of policy selection, the length of $ M $ changes dynamically, due to deleting the policy that does not satisfy the constraints. From Figure \ref{Compa}, we can see that the policy selection lead to dynamic changes in $ |M| $.  Therefore, policy selection plays an important role in the iterative processes.

We also experimentally compare the ToPPO and ToPPO \textit{adapt $ \epsilon $ } in the appendix. According to the Lemma \ref{epsi}, the $ \epsilon $ value is related to the $ N $ value. The value of $ N $ changes at each iteration, and so does $ \epsilon $. ToPPO \textit{adapt $ \epsilon $} represents an adaptive $ \epsilon $ value by the $ N $ value in each iteration. One can see that when $ \epsilon $ is fixed, the performance is better in most environments. The reason for this is that when $\epsilon$ is fixed, it may make the training progress more stable.

Finally, we compare the running times of the different algorithms under the same conditions. It is found that our method takes less time than GePPO and a little more time than PPO, this is because our method, although it does not use the V-trace technique, includes a policy selection step which takes some time. We also tested the effect of choosing a larger $N$ and $\alpha$ on performance. Surprisingly, better results were obtained in some environments, possibly because each environment is affected differently by the hyperparameters. Therefore, fine-tuning the hyperparameters of the algorithm will give better results (see the appendix for more details).

\section{Conclusion and Future Work}\label{conc}

In this paper, we propose a new method Transductive Off-policy PPO (ToPPO), which is a new surrogate objective function method to avoid the inaccuracy of the estimated advantage function of the current policy from off-policy data. Theoretical analysis reveals that the proposed algorithm can guarantee monotonic improvement under certain conditions. Furthermore, our method does not fully utilize off-policy data but selectively utilizes them and gives a theoretical explanation of the trick of the PPO algorithm.
Extensive experimental results show that our proposed method improves the performance compared with several closely related algorithms and also demonstrates the importance of policy selection. 

\paragraph{Limitation.} In this paper, we have found that varying the filter boundary $\alpha$ and the number of trajectories of previous policies $N$ affects the size of the set of selecting policies. And the proposed method performs well by selecting a fixed value of $ \alpha $ and $N$ from the candidate set. For each environment, fine-tuning the hyperparameters $\alpha$ and $N$ will give better results. Now, there is no good way to solve $ \alpha $, and how to dynamically adjust the value of $ \alpha $ in experiments is beneficial to the performance of the algorithm, which will be an interesting direction.

%
%

\bibliographystyle{plain}
\bibliography{ToPPO}

\newpage
\appendix

\section{Appendix}
\subsection{Proof}
Let's start with some useful lemmas.
\begin{lemma}\label{Kak}
	\cite{Kak}
	Consider any two policies $ \tilde{\pi} $ and $ \pi $, we have 
	\begin{equation*}
	\eta(\tilde{\pi}) -\eta(\pi)
	=\frac{1}{1-\gamma}\mathbb{E}_{{s \sim \rho^{\tilde{\pi}},a \sim \tilde{\pi}}} A_{\pi}(s, a).
	\end{equation*}
\end{lemma}

\begin{cor}\label{Kak_cor}
	Consider any two policies $ \tilde{\pi} $ and $ \pi $, we have
	\begin{itemize}
		\item 
		$ V^{\tilde{\pi}}(s_0)-V^{\pi}(s_0)
		=\frac{1}{1-\gamma}\mathbb{E}_{{s \sim \rho^{\tilde{\pi}}(\cdot|s_0),a \sim \tilde{\pi}}} A_{\pi}(s, a). $
		
		\item 
		$ Q^{\tilde{\pi}}(s_0, a_0)-Q^{\pi}(s_0, a_0)
		=\frac{\gamma}{1-\gamma}\mathbb{E}_{{s \sim \rho^{\tilde{\pi}}(\cdot|s_0, a_0),a \sim \tilde{\pi}}} A_{\pi}(s, a) .$
	\end{itemize}
\end{cor}
\begin{proof}
	The first formula is simple, due to $ \eta(\pi)=\mathbb{E}_{s_0\sim\rho_0} V^{\pi}(s_0) $.
	
	Let's proof the second formula.
	\begin{align*}
	&Q^{\tilde{\pi}}(s_0, a_0)-Q^{\pi}(s_0, a_0)\\
	=&\gamma\mathbb{E}_{s'\sim P(s'|s_0, a_0)}\left[V^{\tilde{\pi}}(s')-
	V^{\pi}(s')\right]\\
	=&\frac{\gamma}{1-\gamma}\mathbb{E}_{s'\sim P(s'|s_0, a_0)}\mathbb{E}_{{s \sim \rho^{\tilde{\pi}}(\cdot|s'),a \sim \tilde{\pi}}} A_{\pi}(s, a)\\
	=&\frac{\gamma}{1-\gamma}\mathbb{E}_{{s \sim \rho^{\tilde{\pi}}(\cdot|s_0, a_0),a \sim \tilde{\pi}}} A_{\pi}(s, a). 
	\end{align*}
\end{proof}

\begin{lemma}\label{CPO}
	\cite{Ach}
	Consider two normalized discount state visitation distribution $ \rho^{\tilde{\pi}} $ and $ \rho^{\pi} $, we have
	\begin{align*}
	\|\rho^{\tilde{\pi}}-\rho^{\pi} \|_1
	\leq \frac{\gamma}{1-\gamma} \mathbb{E}_{s\sim\rho^{\pi}}\|\tilde{\pi}-\pi\|_1(s).
	\end{align*}
\end{lemma}

\begin{lemma}\label{pi_A_pi_k}
	Consider any two policies $ \tilde{\pi} $ and $ \pi $, and a advantage function $ A^{\pi} $, we have
	\begin{equation*}
	\|\mathbb{E}_{a \sim \tilde{\pi}} A^{\pi}(s, a)\|_{\infty}\leq
	\max_s \|\tilde{\pi}-\pi\|_1(s)\cdot \|A^{\pi}(s,a)\|_{\infty}.
	\end{equation*}
\end{lemma}
\begin{proof}
	For any $ s $, we have $ \mathbb{E}_{a\sim\pi}A^{\pi}(s, a)=0 $.
	So
	\begin{align*}
	&\|\mathbb{E}_{a \sim \tilde{\pi}} A^{\pi}(s, a)\|_{\infty}\\
	=&\|\mathbb{E}_{a \sim \tilde{\pi}} A^{\pi}(s, a)-\mathbb{E}_{a \sim \pi} A^{\pi}(s, a)\|_{\infty}\\
	=& \|\int_{a\sim \mathcal{A}}(\tilde{\pi}-\pi)A^{\pi}(s, a) da\|_{\infty}\\
	\leq& \max_s \|\tilde{\pi}-\pi\|_1(s)\cdot \|A^{\pi}(s,a)\|_{\infty}.
	\end{align*}
\end{proof}

\textbf{Lemma 2.1.}
(Policy Improvement Lower Bound) Consider a current policy $ \pi_{k} $, and any policies $ \pi $ and $ \mu $, we have
\begin{equation*}
\eta(\pi)-\eta(\pi_k)\geq
\frac{1}{1-\gamma}\mathbb{E}_{(s,a)\sim\rho^{\mu}}\left[\frac{\pi(a|s)}{\mu(a|s)}A^{\pi_k}(s,a)\right]-
\frac{4\epsilon\gamma}{(1-\gamma)^2}\delta_{\max}^{\pi_k, \pi}\cdot\delta^{\pi, \mu},
\end{equation*}
where $\epsilon=\max _{s, a}\left| A^{\pi_k}(s, a)\right|$, $\delta^{\pi, \mu}=\mathbb{E}_{s \sim \rho^{\mu}} \mathbb{D}_{\mathcal{T} \mathcal{V}}(\mu, \pi)(s)$, and $ \delta_{\max}^{\pi_k, \pi} = \max_s \mathbb{D}_{\mathcal{T} \mathcal{V}}(\pi_k, \pi)(s)$. $ \mathbb{D}_{\mathcal{T} \mathcal{V}}(\pi_1, \pi_2)(s) = \frac{1}{2}\sum_{a}|\pi_1(a|s)-\pi_2(a|s)|$ represents the total variation distance (TV) between $ \pi_1(a|s) $ and $ \pi_2(a|s) $ at every state $ s $.
\begin{proof}
	According to lemma \ref{Kak}, we have
	\begin{align*}
	\eta(\pi) -\eta(\pi_k)
	=&\frac{1}{1-\gamma}\mathbb{E}_{{s \sim \rho^{\pi},a \sim \pi}} A^{\pi}(s, a)\\
	=& \frac{1}{1-\gamma}\left(\mathbb{E}_{{s \sim \rho^{\mu},a \sim \pi}} A_{\pi}(s, a) + \int_{s}(\rho^{\pi}-\rho^{\mu}) \mathbb{E}_{a \sim \pi}A^{\pi_k}(s, a)ds\right)\\
	\stackrel{\text { def }\frac{1}{q}+\frac{1}{p}=1}{\geq} &\frac{1}{1-\gamma}\mathbb{E}_{{s \sim \rho^{\mu},a \sim \pi}} A^{\pi_k}(s, a)-\frac{1}{1-\gamma}\|\rho^{\mu}-\rho^{\pi}\|_q\|\mathbb{E}_{a \sim \pi} A^{\pi_k}(s, a)\|_p\\
	\stackrel{\text{when } q=1, p=\infty}{=}&\frac{1}{1-\gamma}\mathbb{E}_{{s \sim \rho^{\mu},a \sim \pi}} A^{\pi_k}(s, a)-\frac{1}{1-\gamma}\|\rho^{\mu}-\rho^{\pi}\|_1\|\mathbb{E}_{a \sim \pi} A^{\pi_k}(s, a)\|_{\infty}.
	\end{align*}
	From lemma \ref{CPO}, lemma \ref{pi_A_pi_k}, and $ \|\tilde{\pi}-\pi\|_1= 2\mathbb{D}_{\mathcal{T} \mathcal{V}}(\pi_1, \pi_2)$, we have
	\begin{align*}
	&\eta(\pi) -\eta(\pi_k)\\
	\geq&\frac{1}{1-\gamma}\mathbb{E}_{{s \sim \rho^{\mu},a \sim \pi}} A^{\pi_k}(s, a)-\frac{1}{1-\gamma}\|\rho^{\mu}-\rho^{\pi}\|_1\|\mathbb{E}_{a \sim \pi} A^{\pi_k}(s, a)\|_{\infty}\\
	\geq & \frac{1}{1-\gamma}\mathbb{E}_{{s \sim \rho^{\mu},a \sim \pi}} A^{\pi_k}(s, a)-\frac{1}{1-\gamma} \frac{\gamma}{1-\gamma} \mathbb{E}_{s\sim\rho^{\mu}}\|\mu-\pi\|_1(s)\max_s \|\pi_k-\pi\|_1(s)\cdot \|A^{\pi_k}(s,a)\|_{\infty}\\
	\geq&
	\frac{1}{1-\gamma}\mathbb{E}_{(s,a)\sim\rho^{\mu}}\left[\frac{\pi(a|s)}{\mu(a|s)}A^{\pi_k}(s,a)\right]-
	\frac{4\epsilon\gamma}{(1-\gamma)^2}\delta_{\max}^{\pi_k, \pi}\cdot\delta^{\pi, \mu},
	\end{align*}
	where $\epsilon=\max _{s, a}\left| A^{\pi_k}(s, a)\right|$, $\delta^{\pi, \mu}=\mathbb{E}_{s \sim \rho^{\mu}} \mathbb{D}_{\mathcal{T} \mathcal{V}}(\mu, \pi)(s)$, and $ \delta_{\max}^{\pi_k, \pi} = \max_s \mathbb{D}_{\mathcal{T} \mathcal{V}}(\pi_k, \pi)(s)$. 
\end{proof}

\textbf{Lemma 3.1}
(Lower Bound) Consider a current policy $ \pi_{k} $, and any policies $ \pi $ and $ \mu $, we have
\begin{equation}
\eta(\pi)-\eta(\pi_k)\geq
\mathcal{L}_{\mu}(\pi)-
\frac{2(1+\gamma)\epsilon}{(1-\gamma)^2}
\delta_{\max}^{\mu, \pi_k}-
\frac{4\epsilon\gamma}{(1-\gamma)^2}
\delta_{\max}^{\mu, \pi}\cdot\delta^{\mu, \pi},
\end{equation}
where $\epsilon=\max _{s, a}\left| A^{\mu}(s, a)\right|$, $\delta^{\mu, \pi}=\mathbb{E}_{s \sim \rho^{\mu}} \mathbb{D}_{\mathcal{T} \mathcal{V}}(\mu, \pi)(s)$, 
$ \delta_{\max}^{\mu, \pi} = \max_s \mathbb{D}_{\mathcal{T} \mathcal{V}}(\mu, \pi)(s)$, and
$ \delta_{\max}^{\mu, \pi_k} = \max_s \mathbb{D}_{\mathcal{T} \mathcal{V}}(\mu, \pi_k)(s)$. 
$ \mathbb{D}_{\mathcal{T} \mathcal{V}}(\pi_1, \pi_2)(s) = \frac{1}{2}\sum_{a}|\pi_1(a|s)-\pi_2(a|s)|$ represents the total variation distance (TV) between $ \pi_1(a|s) $ and $ \pi_2(a|s) $ at every state $ s $.
\begin{proof}
	According to lemma \ref{Kak}, we have
	\begin{align}\label{LB_1}
	\begin{aligned}
	&\eta(\pi) -\eta(\pi_k)\\
	=&\frac{1}{1-\gamma}\mathbb{E}_{{s \sim \rho^{\pi},a \sim \pi}} A_{\pi_k}(s, a)\\
	=&\frac{1}{1-\gamma}\mathbb{E}_{{s \sim \rho^{\pi},a \sim \pi}} A_{\pi_k}(s, a) + \frac{1}{1-\gamma}\mathbb{E}_{{s \sim \rho^{\mu},a \sim \pi}} A_{\mu}(s, a) - \frac{1}{1-\gamma}\mathbb{E}_{{s \sim \rho^{\mu},a \sim \pi}} A_{\mu}(s, a)\\
	=&\frac{1}{1-\gamma}\mathbb{E}_{{s \sim \rho^{\mu},a \sim \pi}} A_{\mu}(s, a) + \frac{1}{1-\gamma}\mathbb{E}_{{s \sim \rho^{\pi},a \sim \pi}} A_{\pi_k}(s, a) - \frac{1}{1-\gamma}\mathbb{E}_{{s \sim \rho^{\pi},a \sim \pi}} A_{\mu}(s, a)\\
	&+ \frac{1}{1-\gamma}\mathbb{E}_{{s \sim \rho^{\pi},a \sim \pi}} A_{\mu}(s, a)-\frac{1}{1-\gamma}\mathbb{E}_{{s \sim \rho^{\mu},a \sim \pi}} A_{\mu}(s, a)\\
	\triangleq& \frac{1}{1-\gamma}\mathbb{E}_{{s \sim \rho^{\mu},a \sim \pi}} A_{\mu}(s, a) + \frac{1}{1-\gamma}\Phi_1 +\frac{1}{1-\gamma}\Phi_2,
	\end{aligned}
	\end{align}
	where $ \Phi_1= \mathbb{E}_{{s \sim \rho^{\pi},a \sim \pi}} A_{\pi_k}(s, a) - \mathbb{E}_{{s \sim \rho^{\pi},a \sim \pi}} A_{\mu}(s, a)$, and $ \Phi_2 =\mathbb{E}_{{s \sim \rho^{\pi},a \sim \pi}} A_{\mu}(s, a)-\mathbb{E}_{{s \sim \rho^{\mu},a \sim \pi}} A_{\mu}(s, a)$.
	
	Next, we prove that $ \Phi_1 $ and $ \Phi_2 $ are bounded, respectively.
	
	According to corollary \ref{Kak_cor} and lemma \ref{pi_A_pi_k}, we have
	\begin{align}\label{LB_2}
	\begin{aligned}
	|\Phi_1|&= |\mathbb{E}_{{s \sim \rho^{\pi},a \sim \pi}} A_{\pi_k}(s, a) - \mathbb{E}_{{s \sim \rho^{\pi},a \sim \pi}} A_{\mu}(s, a)|\\
	&\leq \left\|A^{\pi_k}(s, a)-A^{\mu}(s,a)\right\|_{\infty}\\
	&\leq\left\|Q^{\pi_k}(s, a)-Q^{\mu}(s, a)\right\|_{\infty}+\left\|V^{\pi_k}(s)-V^{\mu}(s)\right\|_{\infty}\\
	&\leq \frac{1+\gamma}{1-\gamma}\left\|\mathbb{E}_{\pi_k}A^{\mu}(s, a)\right\|_{\infty}\\
	& \leq \frac{1+\gamma}{1-\gamma} \max_s \|\pi_k-\mu\|_1(s)\cdot \|A^{\mu}(s,a)\|_{\infty}.
	\end{aligned}
	\end{align}
	According to the proof of the lemma 2.1, we have
	\begin{align}\label{LB_3}
	\begin{aligned}
	|\Phi_2| 
	=|\mathbb{E}_{{s \sim \rho^{\pi},a \sim \pi}} A_{\mu}(s, a)-\mathbb{E}_{{s \sim \rho^{\mu},a \sim \pi}} A_{\mu}(s, a)|
	\leq \frac{4\epsilon\gamma}{1-\gamma}\delta_{\max}^{\mu, \pi}\cdot\delta^{\mu, \pi}.
	\end{aligned}
	\end{align}
	Combining Eqn.(\ref{LB_1}), Eqn.(\ref{LB_2}) and Eqn.(\ref{LB_3}), we can get this conclusion.
\end{proof}

\textbf{Theorem 3.1}
(Monotonic Improvement)
Consider the current policy $ \pi_k $, define
\begin{equation*}
F_{\pi_{k}}(\pi)=\mathcal{L}_{\pi_{k}}(\pi)-\frac{4\epsilon\gamma}{(1-\gamma)^2}\delta_{\max}^{\pi_k, \pi}\cdot\delta^{\pi_k, \pi}.
\end{equation*}
Assume $ \pi_{k+1}= \arg\max_{\pi} F_{\pi_{k}}(\pi)$ exists and satisfies $ F_{\pi_{k}}(\pi_{k+1})>0 $, there exists policy $ \mu $ and constant $ \alpha>0 $ that satisfies $ d(\mu, \pi_{k})<\alpha $, then
\begin{equation*}
\mathcal{L}_{\mu}(\pi_{k+1})-
\frac{2(1+\gamma)\epsilon}{(1-\gamma)^2}
\delta_{\max}^{\mu, \pi_k}-
\frac{4\epsilon\gamma}{(1-\gamma)^2}
\delta_{\max}^{\mu, \pi_{k+1}}\cdot\delta^{\mu, \pi_{k+1}}>0.
\end{equation*}
\begin{proof}
	Let  $\hat{F}_{\mu, \pi_k}(\pi)\triangleq L_\mu (\pi) -
	\frac{2(1+\gamma)\epsilon}{(1-\gamma)^2} \delta_{\max}^{\mu, \pi_k}-
	\frac{4\epsilon\gamma}{(1-\gamma)^2} \delta_{\max}^{\mu, \pi}\cdot\delta^{\mu, \pi}$. 
	We know that  $\hat{F}_{\pi_k, \pi_k}(\pi_{k+1})=F_{\pi_{k}}(\pi_{k+1}) >0$.
	
	Using the locally sign-preserving property of continuous functions  $\hat{F}_{\mu, \pi_k}(\pi)$  about  $\mu$, we know that there exists constant  $\alpha>0$ and policy $\mu$ that satisfies  $d(\mu, \pi_k)<\alpha$ , then $\hat{F}_{\mu, \pi_k}(\pi_{k+1})>0$.
\end{proof}

\textbf{Lemma 4.1}\label{epsi}
Consider clipping parameter $ \epsilon^{\text{PPO}} $ of PPO algorithm and clipping parameter $ \epsilon^{\text{ToPPO}} $ of ToPPO algorithm, $ N=|M| $ presents the number of previous policies of $ M $, we have
\begin{equation}\label{ep_value}
\epsilon^{\text{ToPPO}}=\left\{
\begin{aligned}
&\frac{4}{N+4}\epsilon^{\text{PPO}}, & N\geq 2 \\
&\ \quad\epsilon^{\text{PPO}}\ \ \ \ \ \ , & N =1
\end{aligned}
\right.,
\end{equation}
then at every update the worst-case expected performance loss is the same under both algorithms.
\begin{proof}
	We adopt a similar proof to the paper (Queeney, Paschalidis, and Cassandras 2021).
	Assume $ M=\{\pi_k, \pi_{k-1}, \cdots, \pi_{k-N+1}\} $, we sample a policy $ \pi_{k-i}\in M\setminus\{\pi_k\} $, according to lemma \ref{pi_k_i_pi}, we have
	\begin{align*}
	\mathbb{E}_{s\sim \rho^{\pi_{k-i}}}\mathbb{D}_{\mathcal{T} \mathcal{V}}(\pi_{k-i}, \pi)(s)
	\leq&
	\mathbb{E}_{s\sim \rho^{\pi_{k-i}}}\mathbb{D}_{\mathcal{T} \mathcal{V}}(\pi_{k}, \pi)(s)+
	\sum_{j=1}^{i}
	\mathbb{E}_{s\sim \rho^{\pi_{k-i}}}\mathbb{D}_{\mathcal{T} \mathcal{V}}(\pi_{k-j}, \pi_{k-j+1})(s)\\
	\leq& \frac{\epsilon^{\text{ToPO}}}{2} (i+1).
	\end{align*}
	When $ \mu=\pi_k $, we have 
	$ \mathbb{E}_{s\sim \rho^{\pi_{k}}}\mathbb{D}_{\mathcal{T} \mathcal{V}}(\pi_{k}, \pi)(s)
	\leq \frac{\epsilon^{\text{ToPO}}}{2} $.
	According to Algorithm 1, we know that we use an on-policy data and an off-policy data to update.
	So, the average of the upper bound is $ \frac{1}{2}(\frac{\epsilon^{\text{ToPO}}}{2}+\frac{\epsilon^{\text{ToPO}}}{2}(i+1))= \frac{\epsilon^{\text{ToPO}}}{4}(i+2)$.
	
	Since we are sampling randomly, its expectation is $ \frac{1}{N-1}(3+4+\cdots+(N+1))\frac{\epsilon^{\text{ToPO}}}{4}= \frac{N+4}{2}\frac{\epsilon^{\text{ToPO}}}{4}$.
	
	For PPO, we have $ \mathbb{E}_{s\sim \rho^{\pi_{k}}}\mathbb{D}_{\mathcal{T} \mathcal{V}}(\pi_{k}, \pi)(s)
	\leq \frac{\epsilon^{\text{PPO}}}{2} $.
	
	By comparing the above two formulas, we see that the worst-case expected performance loss for each update is approximately the same for both PPO and ToPPO, when
	\begin{equation*}
	\frac{\epsilon^{\text{PPO}}}{2} = \frac{N+4}{2}\frac{\epsilon^{\text{ToPPO}}}{4}\Longrightarrow
	\epsilon^{\text{ToPPO}}=\frac{4}{N+4}\epsilon^{\text{PPO}}.
	\end{equation*}
\end{proof}

Note that from this Lemma, we know that $ \epsilon^{\text{ToPPO}}\leq\epsilon^{\text{PPO}} $ and $ \epsilon^{\text{ToPPO}} $ decreases monotonically \emph{w.r.t.} $ N\in\{1, 2, 3, \cdots\} $. Given $ N $, $ \epsilon^{\text{ToPPO}} $ is calculated by Eqn.(\ref{ep_value}). In Algorithm 1, we need to keep or delete trajectory data to satisfy constraints. So the true length of $ M $ is not fixed during the training process. When $ \epsilon^{\text{ToPPO}} $ is recalculated according to the true length of $ M $ for each optimization, it has a negative impact on performance (see Figure \ref{Compa_adjust} of the appendix), perhaps because the changing $ \epsilon^{\text{ToPPO}} $ value affects the stability of the training. Hence, we need to fix it. In practice, how do we determine the clipping parameter $ \epsilon^{\text{ToPPO}} $ ? We suggest that it should choose a parameter larger than $ \epsilon^{\text{ToPPO}} $. This is because our algorithm
includes a policy selection step.
During the iteration, we will remove the policy from $M $ that does not satisfy $ \delta^{\mu, \pi_{\theta_{k}}}\leq\alpha $. Hence, the number of previous policies of $ M $ may be less than the initial length $ N $. For example, when $ N=5 $ and $ \epsilon^{\text{PPO}}=0.2 $, we have $ \epsilon^{\text{ToPPO}}=0.089$. We could choose $ \hat{\epsilon}^{\text{ToPPO}}=0.1>0.089 $. We see that Eqn.(\ref{ep_value}) gives the minimum value for choosing the clipping parameter $ \epsilon^{\text{ToPPO}} $, and it may be necessary to adjust the parameter $ \epsilon^{\text{ToPPO}} $ in practice.

\subsection{Reanalyze the PPO algorithm}

\textbf{Theorem }
Consider the current policy $ \pi_k $, define
\begin{equation*}
\pi_{k+1}= \arg\max_{\pi} L_k(\pi),
\end{equation*}
where $ L_k(\pi)=\mathbb{E}_{(s,a)\sim \rho^{\pi_k}}
\min\left(\frac{\pi(a|s)}{\pi_k(a|s)}A^{\pi_k}(s,a),\text{clip}\left(\frac{\pi(a|s)}{\pi_k(a|s)}, 1-\epsilon, 1+\epsilon \right)A^{\pi_k}(s,a)\right) $.
Assume $ \pi_{k+1}$ exists and satisfies $ L_k(\pi_{k+1})\geq 0 $,  then we have
\begin{equation*}
\mathbb{E}_{s\sim\rho^{\pi_{k}}} V^{\pi_{k+1}}(s) \geq \mathbb{E}_{s\sim\rho^{\pi_{k}}} V^{\pi_{k}}(s).
\end{equation*}
\begin{proof}
	Since $ L_k(\pi_{k+1}) \geq 0$, we can get 
	\begin{equation*}
	\mathbb{E}_{(s,a)\sim \rho^{\pi_k}}
	\left(\frac{\pi_{k+1}(a|s)}{\pi_k(a|s)}A^{\pi_k}(s,a)\right)\geq 0.
	\end{equation*}
	This can be reformulated as 
	\begin{equation*}
	\mathbb{E}_{(s, a)\sim \rho^{\pi_k},\pi_{k+1}}
	Q^{\pi_{k}}(s, a)
	\geq
	\mathbb{E}_{(s, a)\sim \rho^{\pi_k},\pi_{k}}
	Q^{\pi_{k}}(s, a).
	\end{equation*}
	Now, using a similar way of the Policy Improvement theorem's proof, we can get
	\begin{align*}
	&\mathbb{E}_{(s, a)\sim \rho^{\pi_k},\pi_{k}}
	Q^{\pi_{k}}(s, a)\\
	\leq&
	\mathbb{E}_{(s, a)\sim \rho^{\pi_k},\pi_{k+1}}
	\mathbb{E}[R(s, a)+\gamma Q^{\pi_{k}}(s', a')|\pi_{k}]\\
	\leq& \mathbb{E}_{(s, a)\sim \rho^{\pi_k},\pi_{k+1}}
	\mathbb{E}[R(s, a)+\gamma R(s',a')+\gamma^2 Q^{\pi_{k}}(s'', a'')|\pi_{k}]\\
	\vdots& \\
	\leq &\mathbb{E}_{(s, a)\sim \rho^{\pi_{k}},\pi^k+1} Q^{\pi_{k+1}}(s, a).
	\end{align*}
	
	Finally, we have
	\begin{equation*}
	\mathbb{E}_{s\sim\rho^{\pi_{k}}} V^{\pi_{k}}(s)
	\leq 
	\mathbb{E}_{s\sim\rho^{\pi_{k}}} V^{\pi_{k+1}}(s).
	\end{equation*}
	
\end{proof}

\begin{figure}[t]
	\centering
	\subfloat[BattleZone]{\includegraphics[width=0.2\textwidth]{Result/atari_per/toppo_evaluate_return_train_BattleZoneNoFrameskip-v4.pdf}}
	\subfloat[Breakout]{\includegraphics[width=0.2\textwidth]{Result/atari_per/toppo_evaluate_return_train_BreakoutNoFrameskip-v4.pdf}}
	\subfloat[Carnival]{\includegraphics[width=0.2\textwidth]{Result/atari_per/toppo_evaluate_return_train_CarnivalNoFrameskip-v4.pdf}}
	\subfloat[Centipede]{\includegraphics[width=0.2\textwidth]{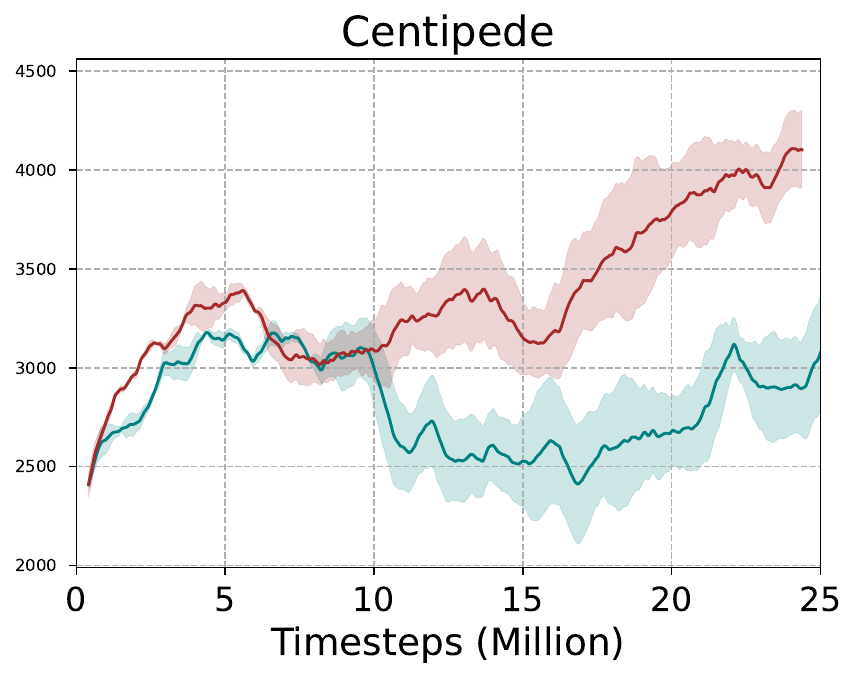}}
	\subfloat[CrazyClimber]{\includegraphics[width=0.2\textwidth]{Result/atari_per/toppo_evaluate_return_train_CrazyClimberNoFrameskip-v4.pdf}}\\
	\subfloat[Enduro]{\includegraphics[width=0.2\textwidth]{Result/atari_per/toppo_evaluate_return_train_EnduroNoFrameskip-v4.pdf}}
	\subfloat[FishingDerby]{\includegraphics[width=0.2\textwidth]{Result/atari_per/toppo_evaluate_return_train_FishingDerbyNoFrameskip-v4.pdf}}
	\subfloat[Gopher]{\includegraphics[width=0.2\textwidth]{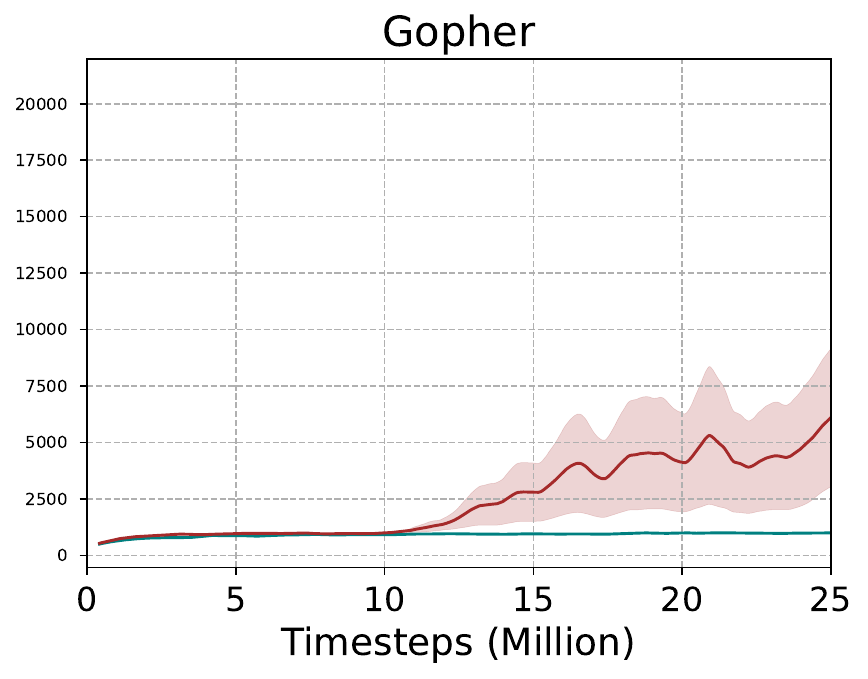}}
	\subfloat[Gravitar]{\includegraphics[width=0.2\textwidth]{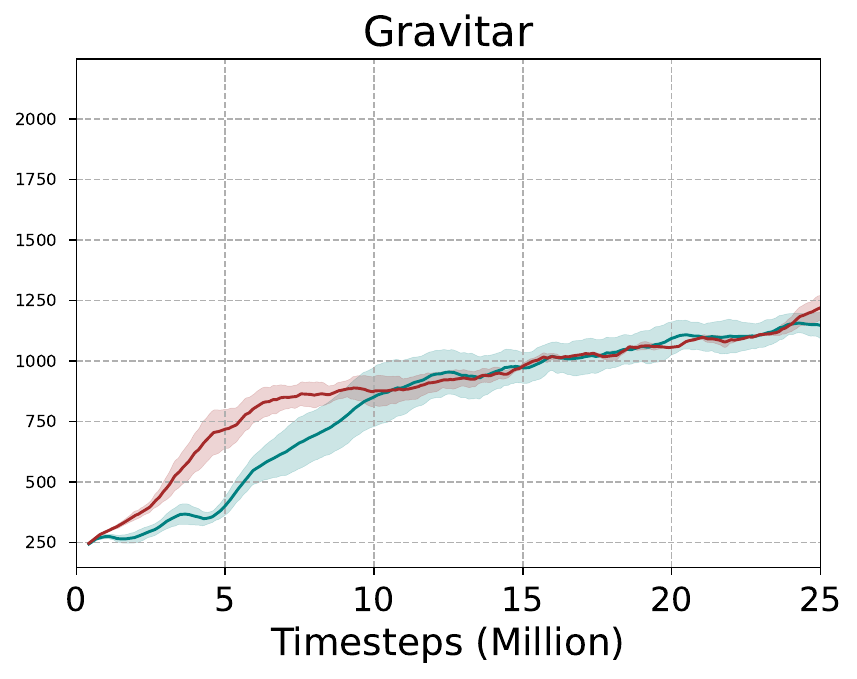}}
	\subfloat[Krull]{\includegraphics[width=0.2\textwidth]{Result/atari_per/toppo_evaluate_return_train_KrullNoFrameskip-v4.pdf}}\\
	\subfloat[KungFuMaster]{\includegraphics[width=0.2\textwidth]{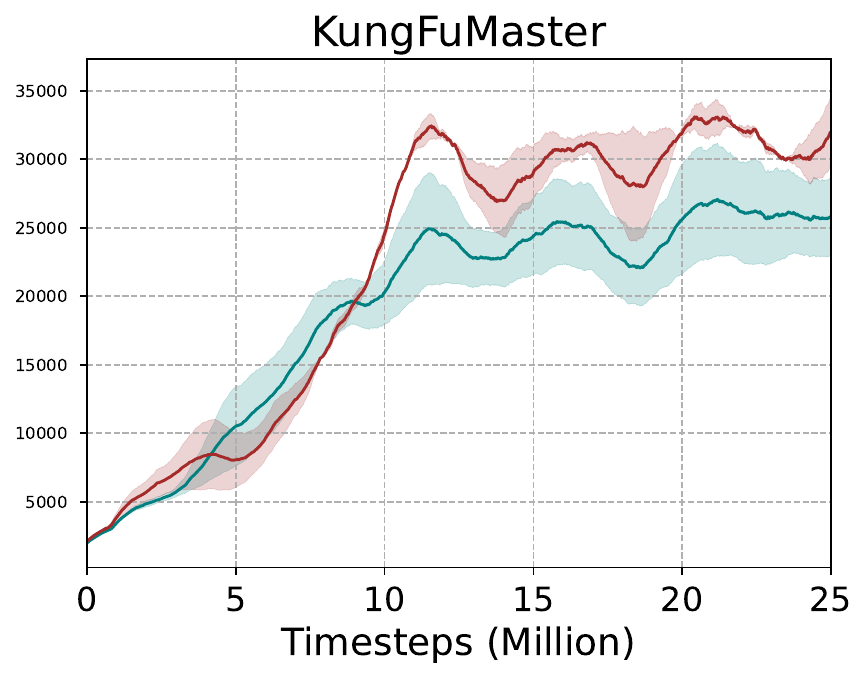}}
	\subfloat[MsPacman]{\includegraphics[width=0.2\textwidth]{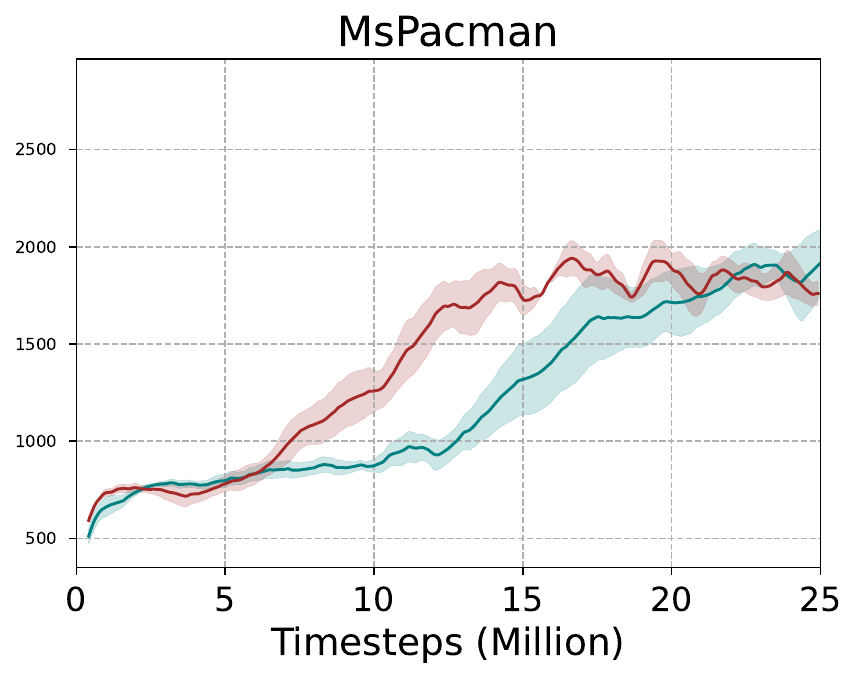}}
	\subfloat[NameThisGame]{\includegraphics[width=0.2\textwidth]{Result/atari_per/toppo_evaluate_return_train_NameThisGameNoFrameskip-v4.pdf}}
	\subfloat[Pooyan]{\includegraphics[width=0.2\textwidth]{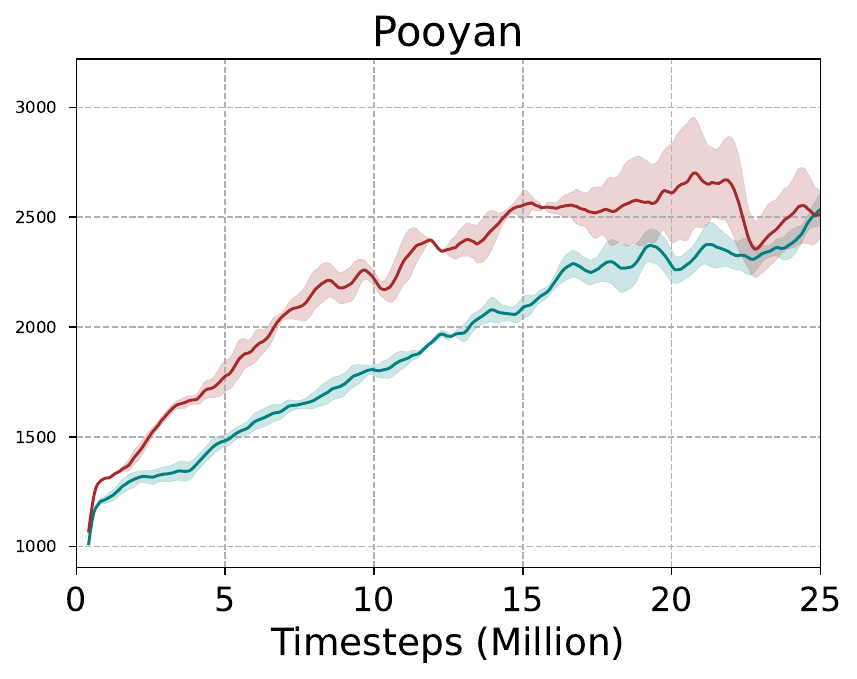}}
	\subfloat[Qbert]{\includegraphics[width=0.2\textwidth]{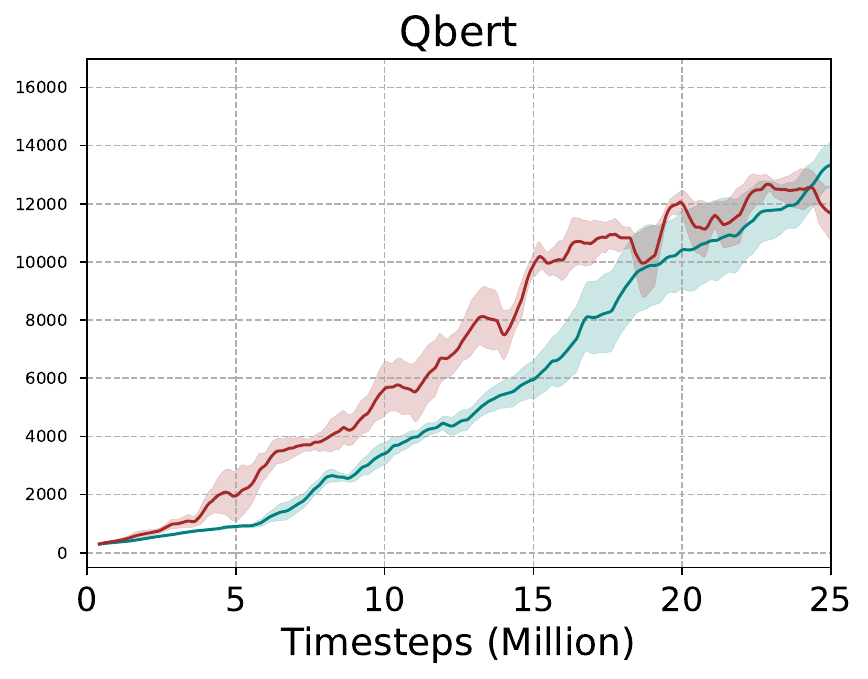}}
	\caption{Learning curves on the some Atari environments. Performance of \textit{ToPPO} vs. \textit{PPO}.
		The shaded region indicates the standard deviation of three random seeds.
	}
	\label{performance_atari1}
\end{figure}

\subsection{Additional experimental results}

\begin{wraptable}{R}{0.55\textwidth}
	\centering
	\caption[width=\columnwidth]{Hyperparameters for ToPPO on Mujoco tasks.}\label{Hy-toppo}
	\begin{tabular}{l|r}
		\toprule
		\multicolumn{1}{l}{Hyperparameter} &
		\multicolumn{1}{r}{ Value } \\
		\midrule
		Discount rate $ \gamma $ & 0.995 \\
		GAE parameter            & 0.97  \\
		Minibatches per epoch    & 32    \\
		Epochs per update        & 10    \\
		Optimizer                & Adam \\
		Learning rate $ \phi $         & 3e-4  \\
		Minimum batch size ($ n $)     & 1024\\
		Number of previous policies ($ N=|M| $) & 5\\
		Clipping parameter $ \epsilon^{\text{ToPPO}} $ & 0.1\\
		Filter boundary $ \alpha $     & 0.03 \\
		\bottomrule
	\end{tabular}
\end{wraptable}


\begin{algorithm}[t]
	\textbf{Input}: Environment $ E $, filter boundary $ \alpha $, discount factor $ \gamma $, batcg size $ n $, clipping parameter $ \epsilon $,\\
	Initialize policy network parameter $ \theta $,\\
	Initialize previous policies data $ M $ and the maximum length of $M$, $  N=|M|$.
	\begin{algorithmic}[0] 
		\FOR{$k=0,1,2,\ldots$}
		\STATE 	\underline{Collect data}:\\
		Collect $n$ samples with $\pi_{\theta_{k}}$ on environment $ E $.\\
		Add $ n $ samples to previous policy data $ M $.
		\STATE \underline{Update policy network}:\\
		Samples a policy data $ \mu $ from set $ M \setminus\{\pi_{\theta_{k}}\} $, and use $n$ samples from the current policy $\pi_{\theta_{k}}$ and $\mu$ (an on-policy data and an off-policy data). \\
		Approximately maximize the empirical objective $L_k (\pi)$ in Eqn.(\ref{off_ppo_loss}) by using stochastic gradient ascent to get new policy network $ \pi_{\theta_{k+1}} $.
		\STATE \underline{Select policies}:\\
		Calculate sample-based estimate $\widehat{\delta}^{\mu, \pi_{\theta_{k+1}}}$ in Eqn.(4), where $ \mu \in M \setminus\{\pi_{\theta_{k+1}}\} $.	
		\IF {$\widehat{\delta}^{\mu, \pi_{\theta_{k+1}}} > \alpha$}
		\STATE  delete $\mu$ data in $ M $
		\ELSE
		\STATE keep $\mu$ data in $ M $	
		\ENDIF
		\ENDFOR
	\end{algorithmic}
	\caption{Transductive Off-policy Proximal Policy Optimization (ToPPO)}\label{ToPPO}
\end{algorithm}

To verify the effectiveness of the proposed ToPPO method, we select seven continuous control tasks from the MuJoCo environments \cite{Tod} in OpenAI Gym \cite{Gre}. We conduct all the experiments mainly based on the code from \cite{Que}. The test procedures are averaged over ten test episodes across ten independent runs. 
The same neural network architecture is used for all methods
The policy network is a Gaussian distribution, and the output of the state-value network is a scalar value. The mean action of the policy network and state-value network are a multi-layer perceptron with hidden layer fixed to [64, 64] and tanh activation \cite{Hen}. The standard deviation of the policy network is parameterized separately \cite{Schtrpo, Schppo}. For the experimental parameters, we use the default parameters from \cite{Duan, Hen}, for example, the discount factor is $ \gamma=0.995 $, and we use the Adam optimizer throughout the training progress. For PPO, the clipping parameter is $ \epsilon^{\text{PPO}}=0.2 $, and the batch size is $ B=2048 $. For GePPO, the clipping parameter is $ \epsilon^{\text{PPO}}=0.1 $, and the batch size of each policy is $ B=1024 $. For TRPO and off-policy TRPO (OTRPO), the bound of trust region is $ \delta=0.01 $, and the batch size of each policy is $ B=1024 $.

\begin{wrapfigure}[]{L}{0.42\textwidth}
	\vskip -5pt
	\centering
	\subfloat[Swimmer]{\includegraphics[width=0.3\textwidth]{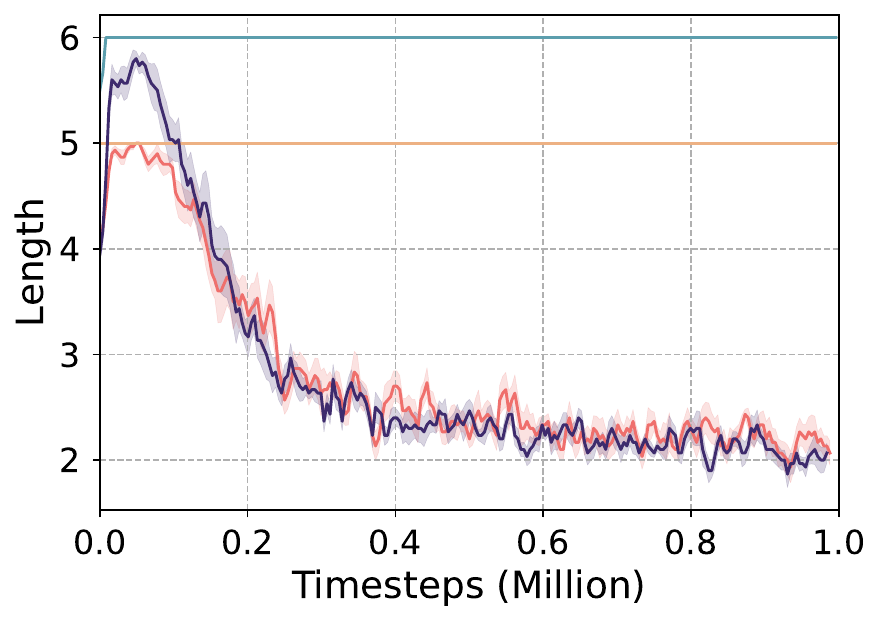}}
	\caption{The number of policies in $ M $. }
	\label{Compa}
	\vspace{-15pt}
\end{wrapfigure}

For our proposed method ToPPO, $ M $ is the replay buffer that stores the old policy trajectories. $ N $ represents the maximum length of $ M $, and the default value is 5. According to the Lemma \ref{epsi}, the clipping parameter $ \epsilon^{\text{ToPPO}} $ is greater than or equal to 0.089, and we set it to $ \epsilon^{\text{ToPPO}}=0.1 $.  The batch size of each policy is $ B=1024 $. Since the TV distance of between the previous policy $ \mu $ and the current policy $ \pi_k $ cannot be calculated exactly, we use the KL divergence $\delta^{\mu, \pi_k}=\mathbb{E}_{s \sim \rho^{\mu}} \mathbb{D}_{KL}(\mu, \pi_k)(s)$ to replace it in practice. The filter boundary is $ \alpha=0.03 $. And we use early stopping trick. 
From Algorithm \ref{ToPPO}, we randomly select a policy $ \mu $ from $ M\setminus\{\pi_{k}\} $, and use $ 1024 $ samples from each of $ \pi_{k} $ and $ \mu $ to update the policy. The experiments are performed on a computer with an Intel Xeon(R) CPU, 64GB of memory and a GeForce RTX 3090 Ti GPU.

\begin{figure}[]
	\begin{minipage}[b]{.99\linewidth}
		\centering
		\subfloat[HalfCheetah]{\includegraphics[width=0.24\textwidth]{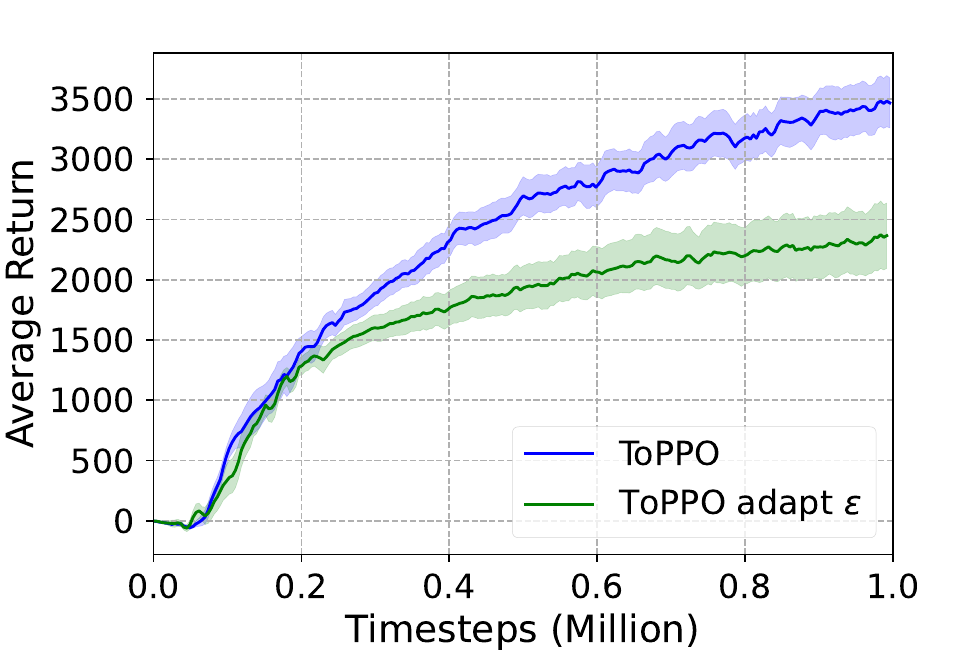}}
		\subfloat[HumanoidStandup]{\includegraphics[width=0.24\textwidth]{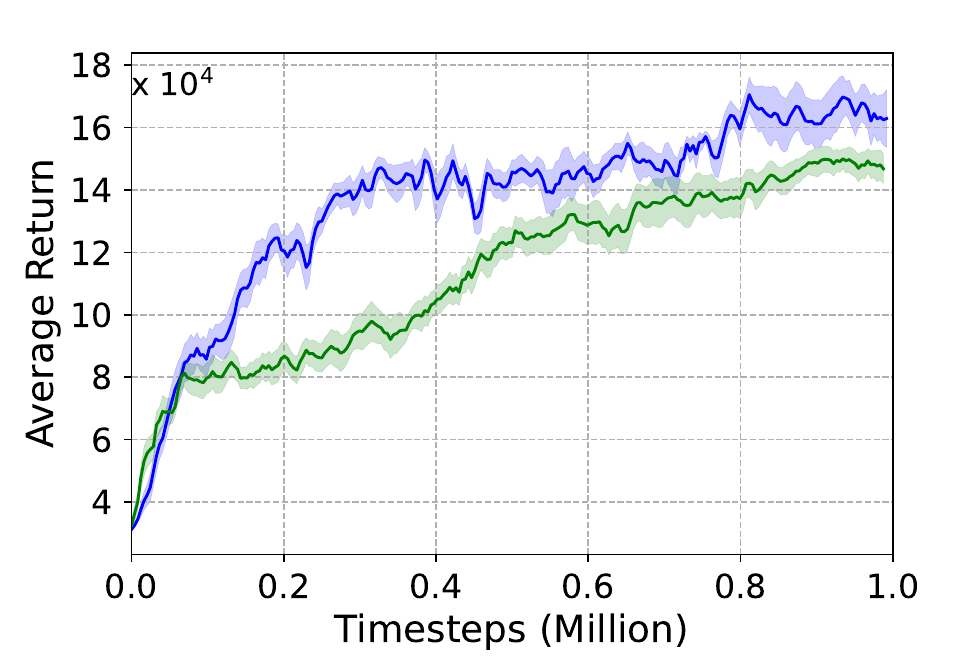}}
		\subfloat[Swimmer]{\includegraphics[width=0.24\textwidth]{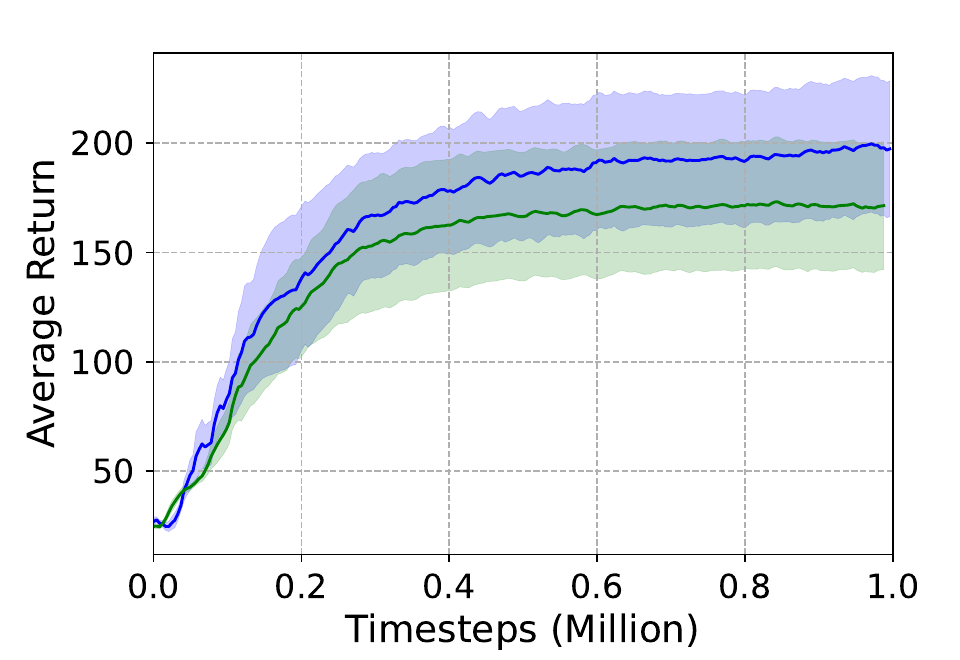}}
		\subfloat[Reacher]{\includegraphics[width=0.24\textwidth]{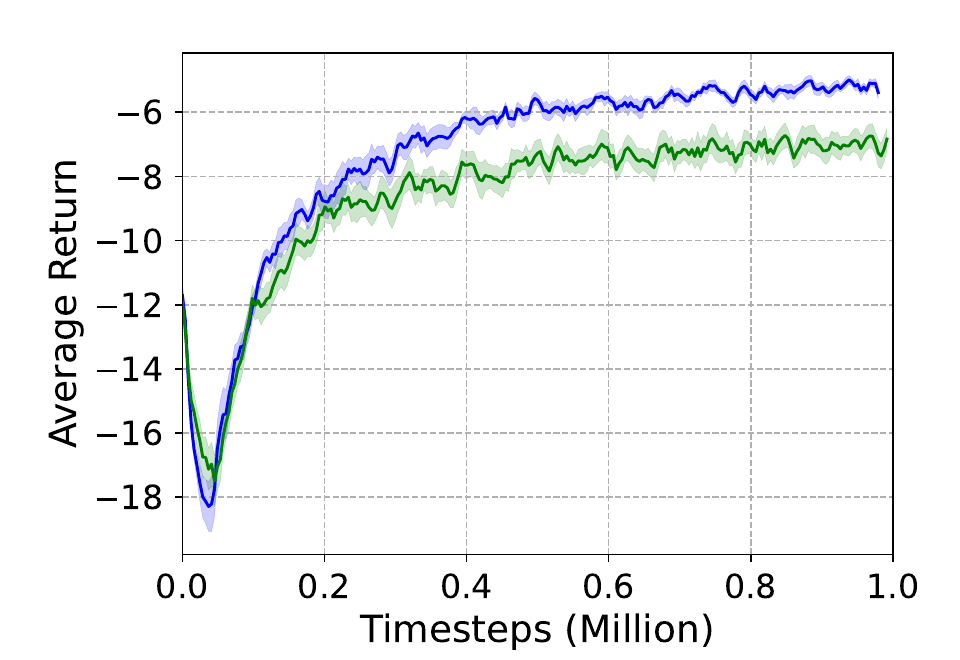}}
	\end{minipage}
	\caption{Learning curves on the Gym environments. 
		ToPPO \textit{adapt} $ \epsilon $ represents an adaptive $ \epsilon $ value by the $ N $ value in each iteration.
		The X-axis represents the timesteps in the environment. 
	}
	\label{Compa_adjust}
\end{figure}

\begin{figure*}[t]
	\centering	
	\subfloat[Swimmer]{\includegraphics[width=0.3\textwidth]{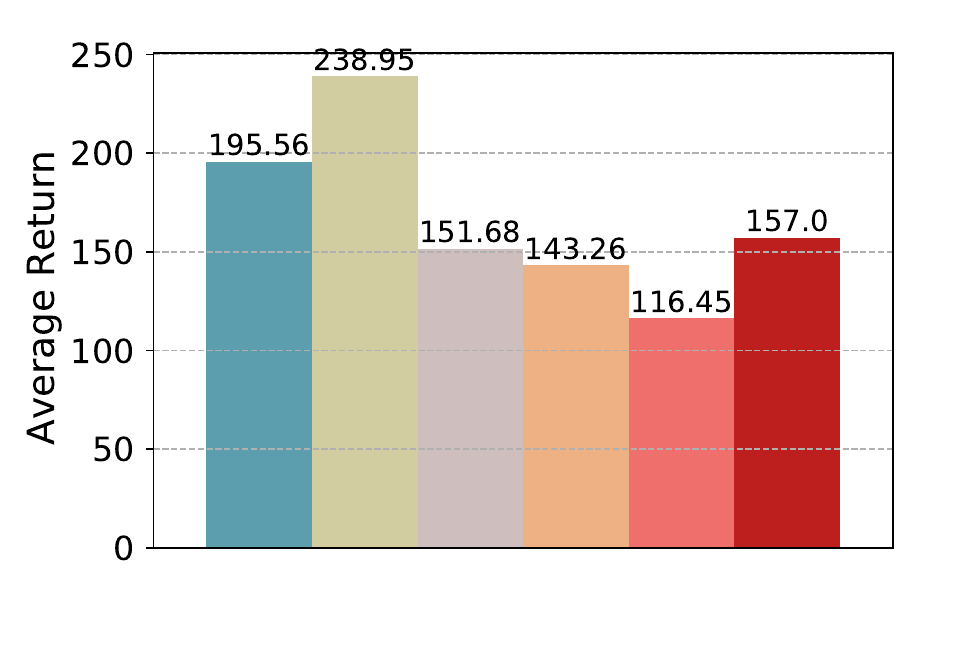}\hskip -5pt}
	\subfloat[HumanoidStandup]{\includegraphics[width=0.3\textwidth]{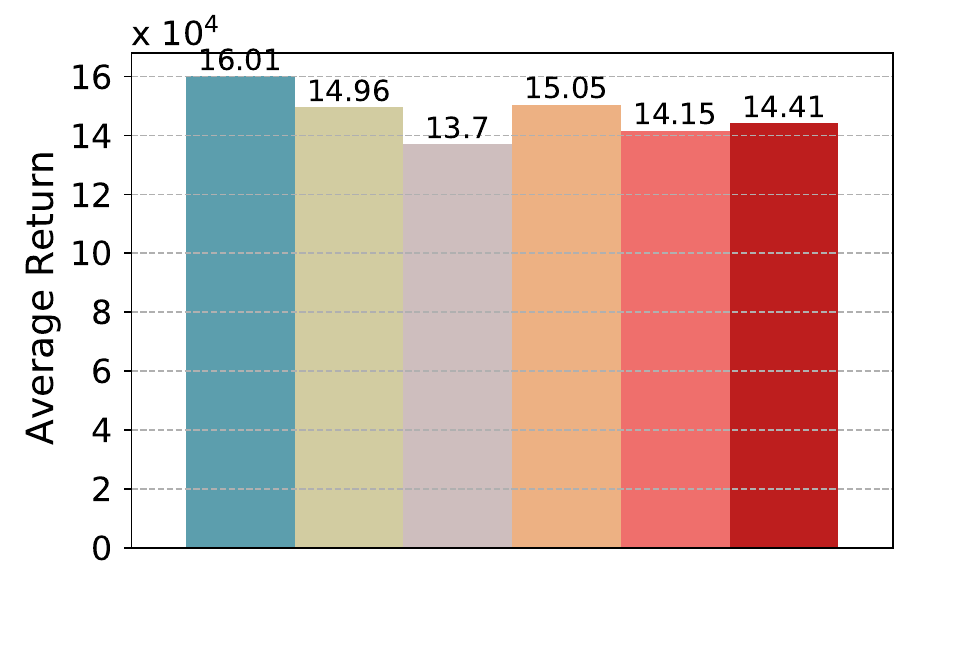}\hskip -5pt}
	\subfloat[Walker2d]{\includegraphics[width=0.3\textwidth]{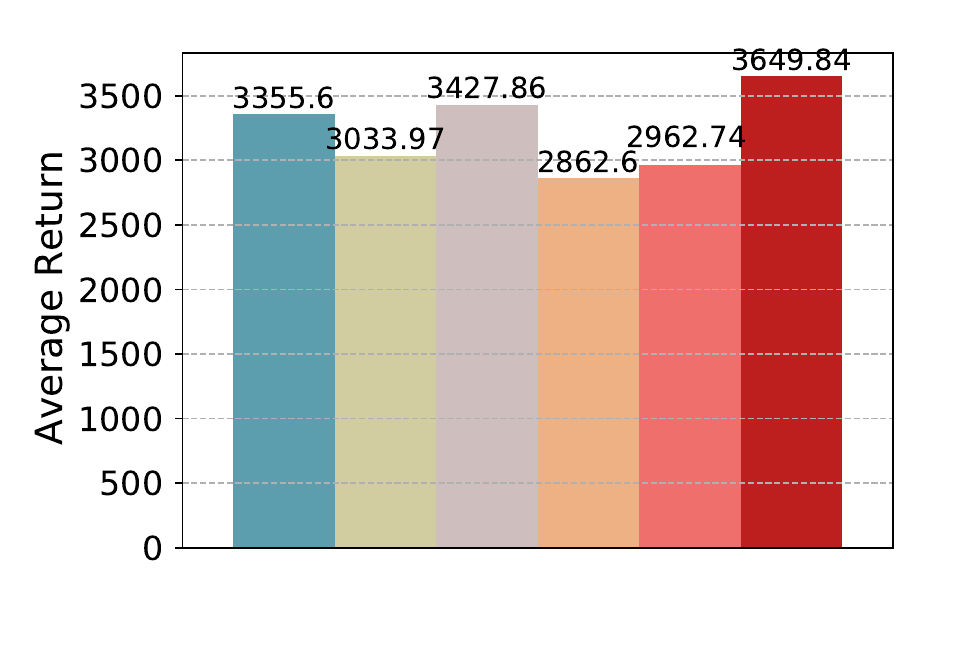}\hskip -5pt}
	\subfloat{\includegraphics[width=0.2\textwidth]{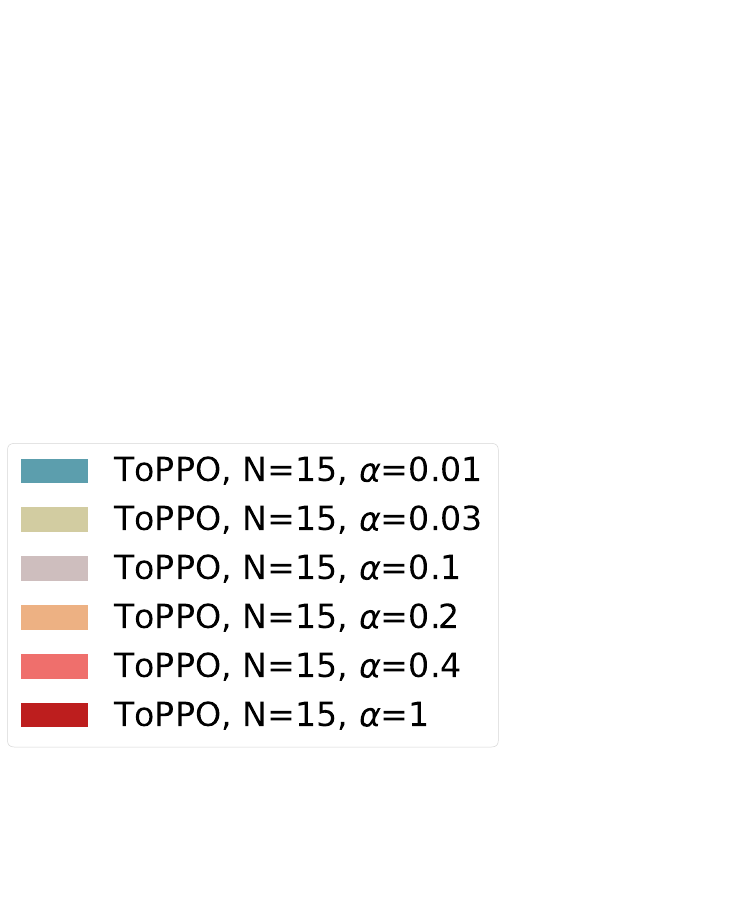}}
	\\	
	\subfloat[Swimmer]{\includegraphics[width=0.3\textwidth]{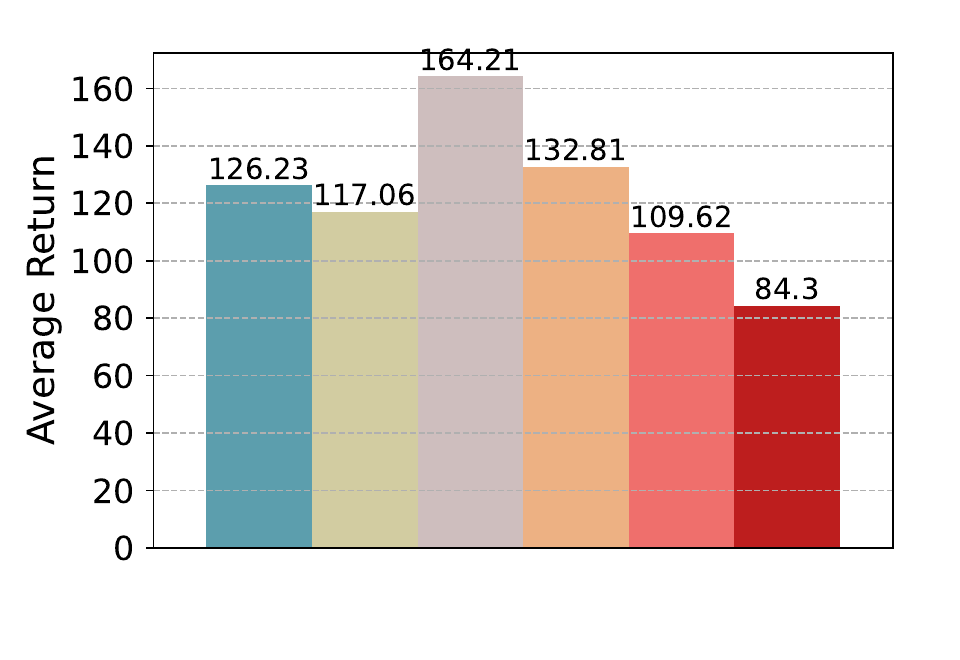}\hskip -5pt}
	\subfloat[HumanoidStandup]{\includegraphics[width=0.3\textwidth]{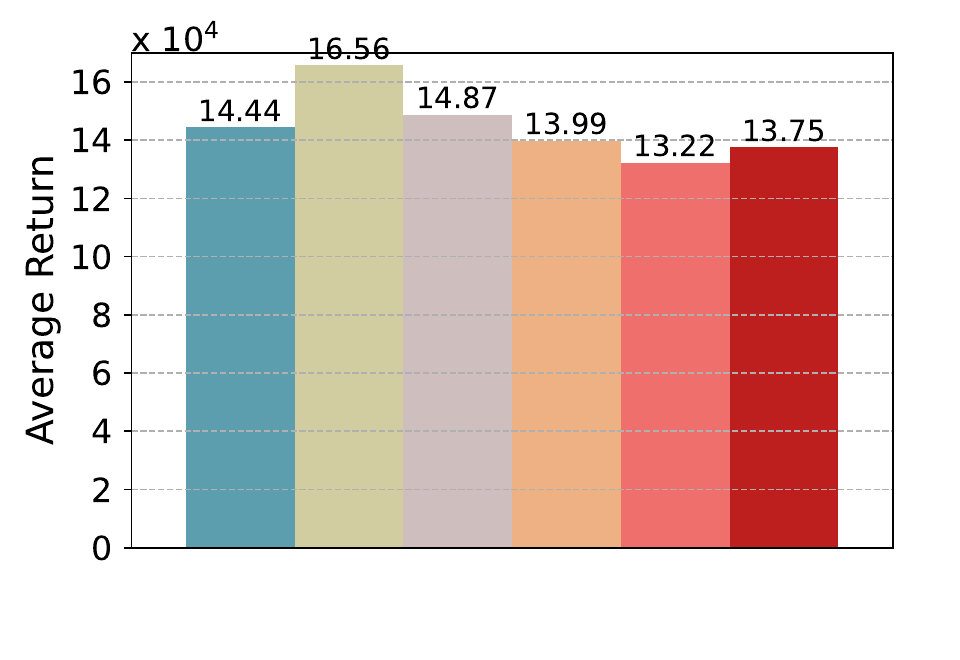}\hskip -5pt}
	\subfloat[Walker2d]{\includegraphics[width=0.3\textwidth]{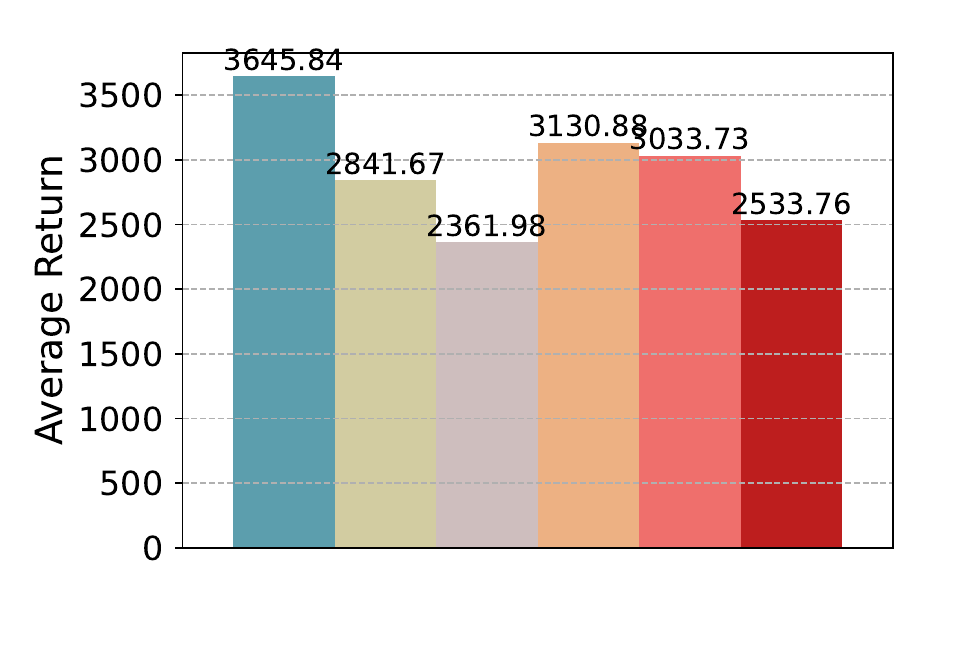}\hskip -5pt}
	\subfloat{\includegraphics[width=0.2\textwidth]{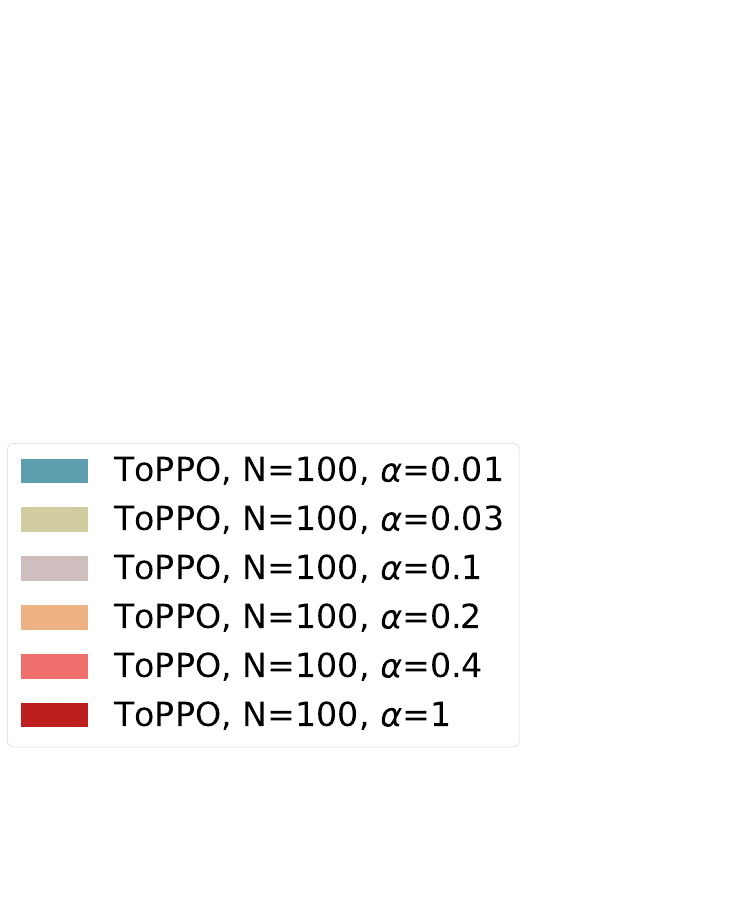}}
	\\
	\caption{Ablation study on the parameter $ \alpha $ under fixed $ N $ in each row in ToPPO. 
		Comparison about the final performance with different parameters.
		Better performance will be obtained if parameter $ \alpha $ is fine-tuned, when fixing $ N $.
		The Y -axis represents the average return.
	}
	\label{performance_N}
\end{figure*}


\end{document}